\documentclass[letterpaper]{article} 
\usepackage{aaai23}  
\usepackage{times}  
\usepackage{helvet}  
\usepackage{courier}  
\usepackage[hyphens]{url}  
\usepackage{graphicx} 
\usepackage{natbib}  
\usepackage{caption} 
\usepackage{algorithm}
\usepackage{algorithmic}

\usepackage{newfloat}
\usepackage{listings}
\DeclareCaptionStyle{ruled}{labelfont=normalfont,labelsep=colon,strut=off} 
\floatstyle{ruled}
\newfloat{listing}{tb}{lst}{}
\floatname{listing}{Listing}
\nocopyright 

\usepackage{subcaption}
\usepackage{booktabs} 
\usepackage{multirow}
\usepackage{nicefrac}       
\usepackage{microtype}      
\usepackage{amsfonts}
\usepackage{amsmath,amsthm}
\newtheorem{assumption}{Assumption}
\newtheorem{lemma}{Lemma}
\newtheorem{proposition}{Proposition}
\newtheorem{corollary}{Corollary}
\usepackage{xcolor}
\usepackage{soul}

\ifodd 0

\newcommand{\com}[1]{{\color{red}{#1}}}
\newcommand{\yl}[1]{{\color{magenta}{#1}}}
\else

\newcommand{\com}[1]{}
\newcommand{\yl}[1]{#1}
\fi

\title{Social Bias Meets Data Bias:\\The Impacts of Labeling and Measurement Errors on Fairness Criteria}
\author {
    Yiqiao Liao\textsuperscript{\rm 1},
    Parinaz Naghizadeh\textsuperscript{\rm 2}
}
\affiliations {
    \textsuperscript{\rm 1} Department of Computer Science and Engineering, The Ohio State University\\
    \textsuperscript{\rm 2} Department of Electrical and Computer Engineering, The Ohio State University\\
    \{liao.489, naghizadeh.1\}@osu.edu
}

\begin{document}

\maketitle

\begin{abstract}
Although many fairness criteria have been proposed to ensure that machine learning algorithms do not exhibit or amplify our existing social biases, these algorithms are trained on datasets that can themselves be statistically biased. In this paper, we investigate the robustness of existing (demographic) fairness criteria when the algorithm is trained on biased data. We consider two forms of dataset bias: errors by prior decision makers in the labeling process, and errors in the measurement of the features of disadvantaged individuals. We analytically show that some constraints (such as Demographic Parity) can remain robust when facing certain statistical biases, while others (such as Equalized Odds) are significantly violated if trained on biased data. We provide numerical experiments based on three real-world datasets (the FICO, Adult, and German credit score datasets) supporting our analytical findings. While fairness criteria are primarily chosen under normative considerations in practice, our results show that naively applying a fairness constraint can lead to not only a loss in utility for the decision maker, but more severe unfairness when data bias exists. Thus, understanding how fairness criteria react to different forms of data bias presents a critical guideline for choosing among existing fairness criteria, or for proposing new criteria, when available datasets may be biased. 
\end{abstract}


\section{Introduction}\label{sec:intro}

Machine learning algorithms are being adopted widely in areas ranging from recommendation systems and ad-display, to hiring, loan approvals, and determining recidivism in courts. Despite their potential benefits, these algorithms can still exhibit or amplify existing societal biases \cite{machine-bias,obermeyer2019dissecting,lambrecht2019algorithmic}. 
This is referred to as algorithmic \emph{(social) bias} or \emph{unfairness}, as the algorithm makes decisions in favor or against individuals in a way that is inconsistent or discriminatory across groups with different social identities (e.g., race, gender). A commonly proposed method for assessing and preventing these forms of unfairness is through \emph{fairness criteria} (e.g, equality of opportunity, equalized odds, or  demographic parity) \citep{mehrabi2021survey,barocas2017fairness}. These criteria typically require the algorithm to make decisions in a way that (approximately) equalizes a statistical measure (e.g., selection rate, true positive rate) between different groups. 

Despite the rising interest in this approach to developing fair algorithms, existing fairness criteria have been largely proposed and evaluated assuming access to unbiased training data. 
However, existing datasets are often themselves \emph{statistically biased} due to biases or errors made during data collection, labeling, feature measurement, etc \cite{blum2020recovering, fogliato2020fairness, jiang2020identifying, kallus2018residual, wick2019unlocking}. Any machine learning algorithm is inevitably only as good as the data it is trained on, 
and so attempts to attain a desired notion of fairness can be thwarted by biases in the training dataset. 

\emph{Our work identifies the impacts of statistical data biases on the efficacy of existing fairness criteria in addressing social biases.} 
In particular, as we show both analytically and numerically, existing fairness criteria differ considerably in their robustness against different forms of statistical biases in the training data. Although fairness criteria are generally chosen under normative considerations in practice, our results show that naively applying fairness constraints can lead to more severe unfairness when data bias exists. Thus, understanding how fairness criteria react to different forms of data bias can serve as a critical guideline when choosing among existing fairness criteria, or when proposing new criteria. 

\paragraph{Overview of our findings and contributions.} Formally, we consider a setting in which a firm makes binary decisions (accept/reject) on agents from two demographic groups $a$ and $b$, with $b$ denoting the disadvantaged group. We assume the training data is statistically biased as a prior decision maker has made errors when either assessing the true qualification state (label) of individuals or when measuring their features. The firm selects its decision rule based on this biased data, potentially subject to one of four fairness criteria: {Demographic Parity (DP)}, {True/False Positive Rate Parity (TPR/FPR)}, or Equalized Odds (EO). Our main findings and contributions are summarized below.

{\emph{(1) Some fairness criteria are more robust than others.}} We first analytically show (Proposition~\ref{prop:f-bias-gamma}) that some existing fairness criteria (namely, DP and TPR) are more robust against labeling errors in the disadvantaged group compared to others (FPR and EO). 
That is, perhaps surprisingly, despite being trained on biased data, the resulting DP/TPR-constrained decision rules continue to satisfy the desired DP/TPR fairness criteria when implemented on unbiased data. {This can be interpreted as a positive byproduct of these fairness criteria, in that (social) fairness desiderata are not violated despite statistical data biases.} 

{\emph{(2) Analysis for different forms of statistical data biases.}} We present similar analyses when the statistical biases are due to feature measurement errors on the disadvantaged group (Proposition~\ref{prop:f-bias-feature}), and labeling biases on the advantaged group (Proposition~\ref{prop:f-bias-gamma-a}). We find that different sets of fairness criteria are robust against different forms of data bias. 

{\emph{(3) Guidelines for the selection of fairness criteria and data debiasing.}} We detail how these observations can be explained based on the effects of each type of data bias on the specific data statistics that a fairness criterion relies on in assessing and imposing its normative desiderata. Our findings can therefore be used to guide targeted data collection and debiasing efforts based on the selected fairness criterion. Alternatively, they could help the decision maker select the most robust fairness criteria among their options, so that it would either continue to be met, or be less drastically impacted, in spite of the suspected types of data bias.

{\emph{(4) Supporting numerical experiments.}} We provide support for our analytical findings through numerical experiments based on three real-world datasets: FICO, Adult, and German credit score (Section~\ref{sec:experiments}). 

{\emph{(5) Fair algorithms may even increase firm utility if the data is biased.}}
Notably, in contrast to the typically discussed ``fairness-accuracy tradeoff'', we show that at times using a fair algorithm can \emph{increase} a firm's expected performance compared to an accuracy-maximizing (unfair) algorithm when training datasets are biased. We highlight this observation in Section~\ref{sec:experiments} and provide an intuitive explanation for it by interpreting fairness constraints as having a regularization effect; we also provide an analytical explanation in Appendix~\ref{app:prop-impact-gamma-proof}.

\paragraph{Related work.} The interplay between data biases and fair machine learning has been a subject of growing interest \citep{ensign2018runaway,neel2018mitigating,bechavod2019equal,kilbertus2020fair,wei2021decision,blum2020recovering,jiang2020identifying, kallus2018residual, rezaei2021robust, fogliato2020fairness, wick2019unlocking}, and our paper falls within this general category. Most of these works differ from ours in that they focus on the sources of these data biases such as feedback loops, censored feedback, and/or adaptive data collection, and on how these exacerbate algorithmic unfairness, how to debias data, and how to build fair algorithms robust to data bias. In contrast, we investigate how existing fair algorithms fare in the face of statistical data biases (without making adjustments to the algorithm or the data collection procedure), and provide potential guidelines for targeting data debiasing efforts accordingly. 

Most closely related to our work are \citep{blum2020recovering, jiang2020identifying, fogliato2020fairness, wick2019unlocking}, which also study the interplay between labeling biases and algorithmic fairness. \citet{jiang2020identifying} propose to address label biases in the data directly, by assigning appropriately selected weights to different samples in the training dataset. \citet{blum2020recovering} study labeling biases in the \emph{qualified disadvantaged group}, as well as re-weighing techniques for debiasing data. Further, they show that fairness intervention in the form of imposing Equality of Opportunity can in fact improve the accuracy achievable on biased training data. \citet{fogliato2020fairness} propose a sensitivity analysis framework to examine the fairness of a model obtained from biased data and consider errors in identifying the \emph{unqualified advantaged group}. \citet{wick2019unlocking} consider errors in identifying both the \emph{unqualified advantaged group and qualified disadvantaged group} together and focus on the fairness accuracy trade-off when applying different approaches to achieve Demographic Parity. In contrast to these works, we contribute through the study of a more comprehensive set of group fairness criteria, as well as a larger set of statistical biases (two types of labeling bias, and feature measurement errors). Our different analysis approach further allows us to provide new insights into which fairness criteria may remain robust (or even help increase a firm's utility), and why, against each form of statistical data bias. 

We review additional related work in Appendix~\ref{app:related}. 

\section{Problem Setting}\label{sec:model}

We analyze an environment consisting of a firm (the decision maker) and a population of agents, as detailed below. Table~\ref{t:notation} in Appendix~\ref{app:notation} summarizes the notation. 

\paragraph{The agents.}  
Consider a population of agents composed of two demographic groups,  distinguished by a sensitive attribute $g\in \{a,b\}$. Let $n_g := \mathbb{P}(G=g)$ denote the  fraction of the population who are in group $g$. 
Each agent has an observable feature $x \in\mathbb{R}$, representing information that is used by the firm in making its decisions; these could be e.g., exam scores or credit scores.\footnote{We consider one-dimensional features (numerical scores) for ease of exposition in our analysis. Our experiments consider both one-dimensional and $n$-dimensional features.} The agent further has a (hidden) binary qualification state $y \in \{0,1\}$, with $y=1$ and $y=0$ denoting those qualified and unqualified to receive favorable decisions, respectively. Let $\alpha_g:=\mathbb{P}(Y=1| G=g)$ denote the qualification rate in group $g$. In addition, let $f^y_g(x) :=\mathbb{P}(X=x| Y=y, G=g)$ denote the probability density function (pdf) of the distribution of features for individuals with qualification state $y$ from group $g$. We make the following assumption on these feature distributions. 
\begin{assumption}\label{ass:MLR}
	The pdfs $f^y_g(x)$ and their CDFs $F^y_g(x)$ are continuously differentiable, and the pdfs satisfy the strict monotone likelihood ratio property, i.e., $\frac{f^1_g(x)}{f^0_g(x)}$ is strictly increasing in $x\in \mathbb{R}$. 
\end{assumption}
This assumption implies that an individual is more likely to be qualified as their feature (score) increases.

We further define the qualification profile of group $g$ as $\gamma_g(x) :=\mathbb{P}(Y=1| X=x, G=g)$, which captures the likelihood that an agent with feature $x$ from group $g$ is qualified. For instance, this could capture estimated repay probabilities given the observed credit scores (which may differ across groups). We let group $b$ be the group with a lower likelihood of being qualified at the same feature ($\gamma_b(x)\leq \gamma_a(x), \forall x$), and refer to it as the disadvantaged group. 

As we show in Section~\ref{sec:analytical}, the firm's optimal decision rule can be determined based on the qualification rates $\alpha_g$ and either one of the other problem primitives: the feature distributions $f^y_g(x)$ or the qualification profiles $\gamma_g(x)$. 
These quantities are related to each other as follows:
\begin{align}
	\gamma_g(x) = \tfrac{f^1_g(x)\alpha_g}{f^1_g(x)\alpha_g+f^0_g(x)(1-\alpha_g)} = \tfrac{1}{1+ \tfrac{f^0_g(x)}{f^1_g(x)}(\frac{1}{\alpha_g}-1)}~.
	\label{eq:gamma-f-alpha}
\end{align}

Existing real-world datasets also often provide information on the qualification rates $\alpha_g$ together with either the feature (distributions) $f^y_g(x)$ (e.g., the UCI Adult dataset) or the qualification profiles $\gamma_g(x)$ (e.g., the FICO credit score dataset); see Section~\ref{sec:experiments}. We will later detail how data biases (in the form of labeling or feature measurement errors) can be viewed as inaccuracies in these measures.

\paragraph{The firm.}  
A firm makes binary decisions $d \in \{0,1\}$ on agents from each group based on their observable features, with $d=0$ and $d=1$ denoting reject and accept decisions, respectively. The firm gains a benefit of $u_+$ from accepting qualified individuals, and incurs a loss of $u_-$ from accepting unqualified individuals. The goal of the firm is to select a (potentially group-dependent) decision rule or policy  $\pi_{g}(x)=\mathbb{P}(D=1| X=x, G=g)$ to maximize its expected payoff. 
In this paper, we restrict attention to threshold policies $\pi_{g}(x) = \mathrm{1}(x\geq \theta_g)$, where $\mathrm{1}(\cdot)$ denotes the indicator function and $\theta_g$ is the decision threshold for group $g$.\footnote{Prior work \cite{liu2018delayed,zhang2020fair} show that threshold policies are optimal under Assumption~\ref{ass:MLR} when selecting fairness-unconstrained policies, and optimal in the fairness-constrained case given additional mild assumptions.} Let $U(\theta_a, \theta_b)=n_aU_a(\theta_a)+n_bU_b(\theta_b)$ denote the firm's expected payoff under policies $\{\theta_a, \theta_b\}$, with $U_g(\theta_g)$ denoting the payoff from group $g$ agents.

The firm may further impose a (group) fairness constraint on the choice of its decision rule. While our framework is more generally applicable, we focus our analysis on \texttt{Demographic Parity (DP)} and \texttt{True/False Positive Rate Parity (TPR/FPR)}.\footnote{We also study \texttt{Equalized Odds (EO)} \cite{hardt2016equality} in our experiments, which requires both TPR and FPR.} 

Let $\mathcal{C}_a^{\texttt{f}}(\theta_a) = \mathcal{C}_b^{\texttt{f}}(\theta_b)$ denote the fairness constraint,\footnote{The choice of hard constraints is for theoretical convenience. In Section~\ref{sec:experiments}, we allow for soft constraints $|\mathcal{C}_a^{\texttt{f}}(\theta_a) - \mathcal{C}_b^{\texttt{f}}(\theta_b)| < \epsilon$.} where $\texttt{f} \in\{\texttt{DP}, \texttt{TPR}, \texttt{FPR}\}$. These constraints can be expressed as follows: 

$\bullet$ \texttt{DP}: This constraint equalizes selection rate across groups, and is given by $C_g^{\texttt{DP}}(\theta) = \int_{\theta}^\infty (\alpha_gf^1_g(x)+(1-\alpha_g)f^0_g(x))\mathrm{d}x$; 

$\bullet$ \texttt{TPR}: Also known as \texttt{Equality of Opportunity} \cite{hardt2016equality}, this constraint equalizes the true positive rate across groups, and can be expressed as  $\mathcal{C}_g^{\texttt{TPR}}(\theta) = \int_{\theta}^{\infty} f^1_g(x)\mathrm{d}x$; 

$\bullet$ \texttt{FPR}: False positive rate parity is defined similarly, with $\mathcal{C}_g^{\texttt{FPR}}(\theta) = \int_{\theta}^{\infty} f^0_g(x)\mathrm{d}x$.

Accordingly, the firm's optimal choice of decision thresholds can be determined by:
\begin{align}
	\max_{\theta_a, \theta_b} & ~ \sum_{g} n_g \int (\alpha_g u_+f^1_g(x) - (1-\alpha_g)u_-f^0_g(x))\pi_g(x)\mathrm{d}x, \notag\\
	\text{s.t.} & \quad \mathcal{C}_a^{\texttt{f}}(\theta_a) = \mathcal{C}_b^{\texttt{f}}(\theta_b)~.
	\label{eq:firm-opt}
\end{align}

Let $\theta_g^{\texttt{f}}$ denote the solution of \eqref{eq:firm-opt} under fairness constraint $\texttt{f}\in\{\texttt{DP}, \texttt{TPR}, \texttt{FPR}\}$, and $\theta_g^{\texttt{MU}}$ denote the \texttt{Maximum Utility (MU)} thresholds (i.e., maximizers of the firm's expected payoff in the absence of a fairness constraint). 

\paragraph{Dataset biases.} 
In order to solve \eqref{eq:firm-opt}, the firm relies on historical information and training datasets to obtain estimates of the underlying population characteristics: the qualification rates $\alpha_g$, the feature distributions $f_g^y(x)$, and/or the qualification profiles $\gamma_g(x)$. However, the estimated quantities $\hat{\alpha}_g$, $\hat{f}_g^y(x)$, and/or $\hat{\gamma}_g(x)$ may differ from the true population characteristics. 
We refer to the inaccuracies in these estimates as data bias. Specifically, we focus our analysis on the following instantiations of our general model:

    \paragraph{1. Qualification assessment (labeling) biases,} reflected in the form of errors in profiles $\gamma_g(x)$. We note that such biases can also affect the estimate of $\alpha_g$ and $f^y_g(x)$. This case is most similar to labeling biases considered in prior work \cite{blum2020recovering,jiang2020identifying,fogliato2020fairness, wick2019unlocking}. Here, we consider two specific forms of this type of bias:
    
    \begin{itemize}
        \item \emph{Flipping labels on qualified disadvantaged agents.} We first consider biases that result in $\hat{\gamma}_b(x) = \beta \gamma_b(x), \forall x$, where $\beta\in(0,1)$ is the underestimation rate. This can be viewed as label biases due to a prior decision maker/policy that only had a probability $\beta<1$ of correctly identifying qualified agents from the disadvantaged group $b$. We start with this type of bias as it is one of the most difficult to rectify. Specifically, these biases will not be corrected post-decision due to censored feedback: once a qualified agent is labeled as 0 and rejected, the firm does not get the opportunity to observe the agent and assess whether this was indeed the correct label.  
        \item \emph{Flipping labels on unqualified advantaged agents.} We also consider biases of the form of $\hat{\gamma}_a(x)=(1-\beta){\gamma_a(x)} + \beta$, with $\beta\in(0,1)$, interpreted as prior errors by a decision maker who mistakenly labeled unqualified agents from the advantaged group as qualified with probability $\beta$.
    \end{itemize}
    
    \paragraph{2. Feature measurement errors,} in the form of drops in the feature distribution likelihood ratios in the disadvantaged group. Formally, we consider biases that result in $\frac{\hat{f}^1_b({x})}{\hat{f}^0_b(x)} = \beta(x)  \frac{f_b^1(x)}{f_b^0(x)}, \forall x$, where $\beta(x):\mathbb{R}\rightarrow (0,1)$ is the underestimation rate and is a non-decreasing function in $x$ (including constant). In words, this results in a firm assessing that an agent with a given feature $x$ is less likely to be qualified than it truly is. 
    This type of bias can occur, for instance, if scores are normally distributed and systematically underestimated such that $\hat{x}=x-\epsilon$, where $\epsilon\geq 0$. 
    This case generalizes measurement biases studied in \cite{liu2018delayed}.

Note that both the firm's expected payoff (objective function in \eqref{eq:firm-opt}) and the fairness criteria are impacted by such data biases. 
In the next sections, we analyze, both theoretically and numerically, the impacts of these types of data biases on the firm's ability to satisfy the desired fairness metric \texttt{f}, as well as on the firm's expected payoff. 

\section{Analytical Results}\label{sec:analytical} 

We begin by characterizing the decision thresholds $\theta_g^{\texttt{MU}}$ that maximize the firm's expected utility, in the absence of any fairness constraints, and investigate the impacts of data biases on these thresholds. All proofs are given in Appendix~\ref{app:analytical-proofs}.

\begin{lemma}[Optimal \texttt{MU} thresholds]\label{lemma:MU-opt}
The thresholds $\{\theta^\texttt{MU}_a, \theta^\texttt{MU}_b\}$ maximizing the firm's utility satisfy $\frac{f^1_g(\theta_g^{\texttt{MU}})}{f^0_g(\theta_g^{\texttt{MU}})}=\frac{(1-\alpha_g)u_-}{\alpha_g u+}$. Equivalently, $\gamma_g(\theta_g^{\texttt{MU}})=\frac{u_-}{u_+ + u_-}$.
\end{lemma}

\begin{lemma}[Impact of data biases on \texttt{MU} thresholds and firm's utility]\label{lemma:MU-bias}
Let ${\theta}^{\texttt{MU}}_g$ and $\hat{\theta}^{\texttt{MU}}_g$ denote the optimal \texttt{MU} decision thresholds for group $g$, obtained given unbiased data and data with biases on group $b$, respectively. If
    (i) $\hat{\gamma}_b({\theta}^{\texttt{MU}}_b) < {\gamma_b({\theta}^{\texttt{MU}}_b)}$, or, 
    (ii) $\frac{\hat{f}^1_b({\theta}^{\texttt{MU}}_b)}{\hat{f}^0_b({\theta}^{\texttt{MU}}_b)} <  \frac{f_b^1({\theta}^{\texttt{MU}}_b)}{f_b^0({\theta}^{\texttt{MU}}_b)}$, 
then the decision threshold on group $b$ increases, i.e., $\hat{\theta}^{\texttt{MU}}_b>{\theta}^{\texttt{MU}}_b$. 
The reverse holds if the inequalities above are reversed. In all these cases, the decisions on group $a$ are unaffected, i.e., $\hat{\theta}^{\texttt{MU}}_a = {\theta}^{\texttt{MU}}_a$. Further, the firm's utility decreases in all cases, i.e., $U({\theta}^{\texttt{MU}}_a, \hat{\theta}^{\texttt{MU}}_b) < U({\theta}^{\texttt{MU}}_a, {\theta}^{\texttt{MU}}_b)$. 
\end{lemma}

As intuitively expected, biases against the disadvantaged group (underestimation of their qualification profiles, or scores) lead to an increase in their disadvantage; the reverse is true if the group is perceived more favorably. We also note that the decisions on group $a$ remain unaffected by any biases in group $b$'s data. This implies that if the representation of group $b$ is small, the firm has less incentive for investing resources in removing data biases on group $b$. In the remainder of this section, we will show that the coupling introduced between the group's decisions due to fairness criteria breaks this independence. A takeaway from this observation is that once a fairness-constrained algorithm couples the decision rules between two groups, it also makes statistical data debiasing efforts advantageous to \emph{both} groups, and therefore increases a (fair) firm's incentives for data debiasing.

The next lemma characterizes the optimal fairness-constrained decision thresholds. 

\begin{lemma}[Optimal fair thresholds]\label{lemma:fair-opt}
The thresholds $\{\theta^\texttt{f}_a, \theta^\texttt{f}_b\}$ maximizing the firm's expected utility subject to a fairness constraint 
$\texttt{f}\in\{\texttt{DP}, \texttt{TPR}, \texttt{FPR}\}$ satisfy 
$\sum_{g} n_g \tfrac{\alpha_g u_+f^1_g(\theta^\texttt{f}_g) -  (1-\alpha_g)u_-f^0_g(\theta^\texttt{f}_g)}{\partial C_g^{\texttt{f}}(\theta^\texttt{f}_g)/\partial \theta} = 0$.
\end{lemma}

This characterization is similar to those obtained in prior works \cite{zhang2019group,zhang2020fair} (our derivation technique is different).  
Using Lemma~\ref{lemma:fair-opt}, we can further characterize the thresholds for fairness criteria  $\texttt{f}\in\{\texttt{DP}, \texttt{TPR}, \texttt{FPR}\}$, as derived in Tables~\ref{t:fair-thds} and \ref{t:fair-thds-2} in  Appendix~\ref{app:table-derivations}.  These form the basis of the next set of results, which shed light on the sensitivity of different fairness criteria to biased training data. 

\subsection{Impacts of Labeling Biases}\label{sec:gamma-bias}
We now assess the sensitivity of fairness-constrained policies to biases in qualification assessment (labeling). We first consider biases that result in $\hat{\gamma}_b(x) = \beta \gamma_b(x), \forall x$, where $\beta\in(0,1]$ is the underestimation rate (this could be due to, e.g., labeling errors on qualified individuals from the disadvantaged group). 
We first analyze the impacts of such biases on the decision thresholds and on the firm's utility. 

\begin{proposition}
\label{prop:f-bias-gamma}
Assume the qualification profile of group $b$ is underestimated so that $\hat{\gamma}_b(x) = \beta \gamma_b(x), \forall x$, where $\beta\in(0,1]$. 
Let ${\theta}^{\texttt{f}}_g$ and $\hat{\theta}^{\texttt{f}}_g(\beta)$ denote the optimal decision thresholds satisfying fairness constraint $\texttt{f}\in\{\texttt{DP}, \texttt{TPR}, \texttt{FPR}\}$, obtained from unbiased data and from data with biases on group $b$ given $\beta$, respectively. Then,

(i) $\hat{\theta}^{\texttt{f}}_g(\beta)\geq {\theta}^{\texttt{f}}_g$  for $g\in\{a,b\}, ~\texttt{f}\in\{\texttt{DP}, \texttt{TPR}, \texttt{FPR}\}, ~ \beta\in(0,1]$. Further, $\hat{\theta}^{\texttt{f}}_g(\beta)$ is decreasing in $\beta$. 
    
(ii) The \texttt{DP} and \texttt{TPR}  criteria continue to be met, while \texttt{FPR} is violated, at their $\{\hat{\theta}^{\texttt{f}}_a(\beta), \hat{\theta}^{\texttt{f}}_b(\beta)\}$.
    
(iii)    The firm's utility decreases under $\texttt{f}\in\{\texttt{DP}, \texttt{TPR}\}$, i.e.,  $U(\hat{\theta}^{\texttt{f}}_a(\beta), \hat{\theta}^{\texttt{f}}_b(\beta))\leq U({\theta}^{\texttt{f}}_a, {\theta}^{\texttt{f}}_b)$.
    
(iv) The firm's utility may increase or decrease under $\texttt{FPR}$. 

\end{proposition}
\begin{figure}
\centering
\includegraphics[width=0.45\textwidth]{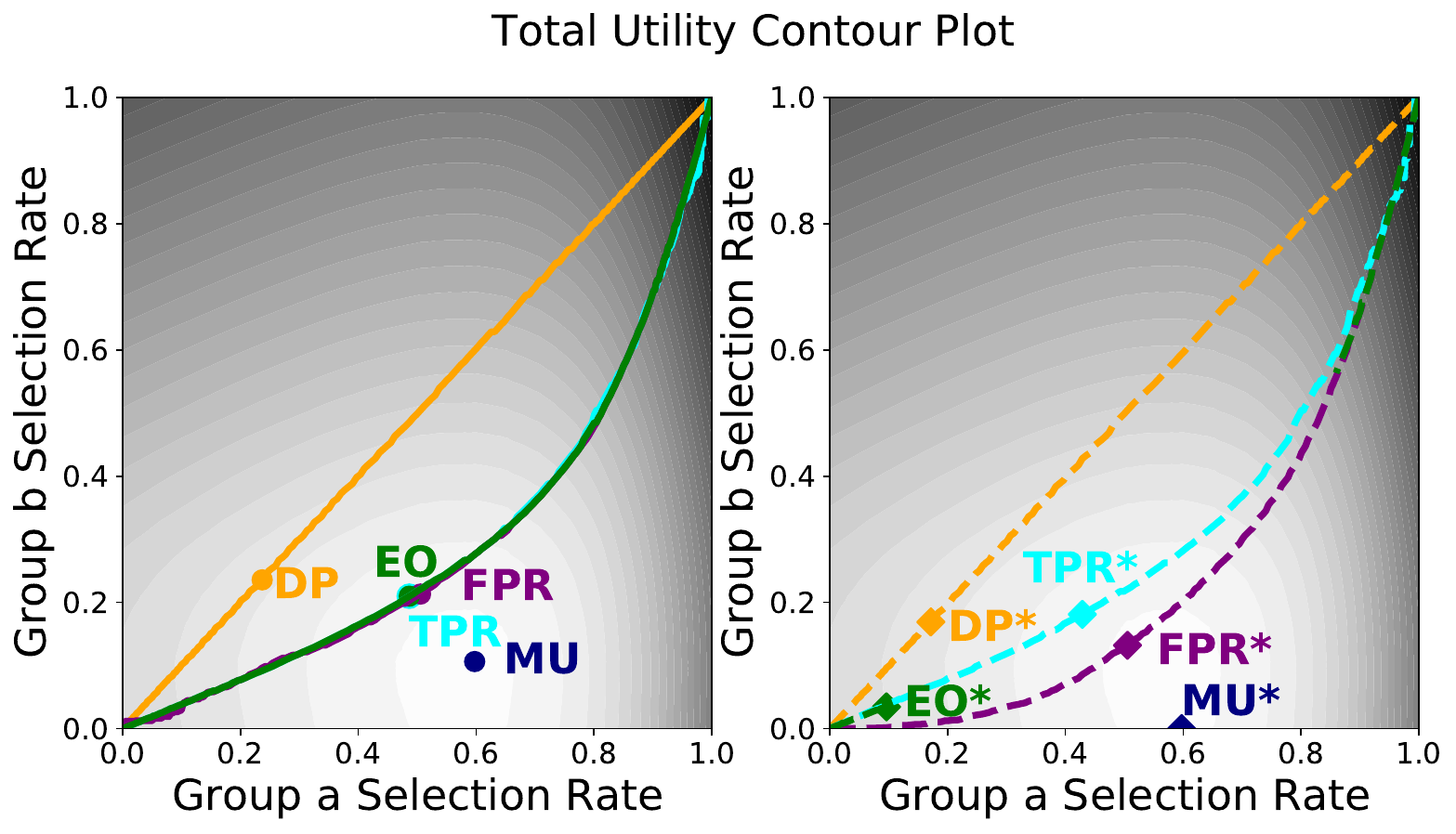}
\caption{Firm's utility as a function of selection rates. Brighter regions represent higher utility. The curves highlight solutions satisfying the fairness constraints. Left: unbiased data. Right: biased data with 20\% of the qualification states of the qualified agents from group $b$ flipped from 1 to 0.}
\label{fig:synthetic-label-util-contour-prop-1}
\end{figure}
Figure~\ref{fig:synthetic-label-util-contour-prop-1} illustrates Proposition~\ref{prop:f-bias-gamma} on a synthetic dataset inspired by the FICO dataset~\cite{hardt2016equality}. 

The proof of this proposition relies on the characterizations of the optimal fairness-constrained thresholds in Lemma~\ref{lemma:fair-opt}, together with identifying the changes in the problem primitives when  $\hat{\gamma}_b(x) = \beta \gamma_b(x)$. In particular, we show that $\hat{\alpha}_b(x)=\beta\alpha_b(x)$, $\hat{f}^1_b(x)=f^1_b(x)$, and $\hat{f}^0_b(x)=\frac{1-\alpha_b}{1-\beta\alpha_b} \frac{1-\beta\gamma_b(x)}{1-\gamma_b(x)} f^0_b(x)$. Intuitively, these changes can be explained as follows: $\hat{\gamma}_b(x)=\beta \gamma_b(x)$ can be viewed as flipping label 1 to label 0 in the training data on group $b$ with probability $\beta$. This leaves the feature distribution of qualified agents unchanged, whereas it adds (incorrect) data on unqualified agents, hence biasing $f^0_b(x)$. Such label flipping also decreases estimated qualification rates $\alpha_b$ by a factor $\beta$. Using these, we show that \texttt{DP/TPR} continue to hold at the biased thresholds given the changes in the statistics they rely on, while \texttt{FPR} is violated. As we show later, the impacts of other types of data bias, and the robustness of any fairness criteria against them, can be similarly tracked to the impacts of those statistical data biases on different data statistics.

We next note two main differences of this lemma with Lemma~\ref{lemma:MU-bias} in the unconstrained setting: (1) the biases in group $b$'s data now lead to under-selection of \emph{both} groups compared to the unbiased case. That is, the introduction of fairness constraints couples the groups in the impact of data biases as well. (2) Perhaps more interestingly, there exist scenarios in which the adoption of a fairness constraint \emph{benefits} a firm facing biased qualification assessments. (We provide additional intuition for this in Appendix~\ref{app:prop-impact-gamma-proof}). Note however that the fairness criterion is no longer satisfied in such scenarios.

In addition, Proposition~\ref{prop:f-bias-gamma} shows that the \texttt{DP} and \texttt{TPR} fairness criteria are \emph{robust} to underestimation of qualification profiles of the disadvantaged group, in that the obtained thresholds continue to satisfy the desired notion of fairness. That said, the proposition also states that the pair of decision thresholds $\{\theta^\texttt{f}_a, \theta^\texttt{f}_b\}$ are different from (and higher than) those that would be obtained if data was unbiased, and hence lead to the loss of utility for the firm. To better assess the impacts of these changes on the firm's expected payoff, we investigate the sensitivity of \texttt{DP} and \texttt{TPR} thresholds to the error rate $\beta$. 

Formally, we can use the results of Lemma~\ref{lemma:fair-opt} together with the implicit function theorem to characterize $\tfrac{\partial \hat{\theta}^{\texttt{f}}_g(\beta)}{\partial \beta}$, the rates of change in the thresholds as a function of the underestimation rate $\beta$; see Proposition~\ref{prop:f-bias-gamma-sensitivity} in Appendix~\ref{app:f-bias-gamma-sensitivity}. 
The following corollary of that proposition shows that, under mild conditions, \texttt{DP} is more sensitive to qualification assessment biases than \texttt{TPR} when facing the same bias rates. 

\begin{corollary}
\label{cor:DP-sensitive} Consider $\frac{\partial \hat{\theta}^{\texttt{f}}_b(1)}{\partial \beta}$, the rate of change of group $b$'s thresholds at $\beta=1$. There exists a $\bar{\alpha}_b$ such that for all $\alpha_b\leq \bar{\alpha}_b$, we have  $|\frac{\partial \hat{\theta}^{\texttt{TPR}}_b(1)}{\partial \beta}| <  |\frac{\partial \hat{\theta}^{\texttt{DP}}_b(1)}{\partial \beta}|$; that is, \texttt{DP} is more sensitive to qualification assessment biases than \texttt{TPR}.
\end{corollary}

This higher sensitivity of \texttt{DP} compared to \texttt{TPR} leads to a higher drop in the utility of the firm when \texttt{DP}-constrained classifiers are used on such biased datasets; we further illustrate this in our experiments in Section~\ref{sec:experiments}.

{Finally, our analysis can be similarly applied to study the impacts of labeling biases on the \emph{advantaged group}; we detail this analysis in Appendix~\ref{app:gamma-biases-advantaged}. In particular, we consider biases of the form of $\hat{\gamma}_a(x)=(1-\beta){\gamma_a(x)} + \beta$, with $\beta\in[0,1)$, interpreted as prior errors by a decision maker who mistakenly labeled unqualified individuals from the advantaged group as qualified with probability $\beta$.} In Proposition~\ref{prop:f-bias-gamma-a}, we show that this time, \texttt{DP} and \texttt{FPR} are robust against these biases, while \texttt{TPR} is in general violated. Notably, \texttt{DP} remains robust against \emph{both} types of qualification assessment bias. {Our experiments in Section~\ref{sec:experiments} further support this observation by showing that \texttt{DP}-constraint thresholds are more robust to label flipping biases induced in different real-world datasets.}

\subsection{Impacts of Feature Measurement Errors}\label{sec:f-bias} 
We now analyze the sensitivity of fairness-constrained decisions to an alternative form of statistical biases: errors in feature measurements of the disadvantaged group. 

\begin{proposition}
\label{prop:f-bias-feature}
Assume the features of group $b$ are incorrectly measured, so that $\frac{\hat{f}^1_b({x})}{\hat{f}^0_b(x)} = \beta(x)  \frac{f_b^1(x)}{f_b^0(x)}, \forall x$, where $\beta(x):\mathbb{R}\rightarrow (0,1)$ is a non-decreasing function. Let ${\theta}^{\texttt{f}}_g$ and $\hat{\theta}^{\texttt{f}}_g(\beta)$ denote the optimal decision thresholds satisfying fairness constraint $\texttt{f}\in\{\texttt{DP}, \texttt{TPR}, \texttt{FPR}\}$, obtained from unbiased data and data with biases on group $b$ with error function $\beta$, respectively. Then, 

~~ (i) If $\hat{f}^1_b(x)=f^1_b(x), \forall x$ (resp. $\hat{f}^0_b(x)=f^0_b(x), \forall x$), \texttt{TPR} (resp. \texttt{FPR}) will be met at the new thresholds. 

~~ (ii) If $\hat{F}^1_b(x)<F^1_b(x), \forall x\geq {\theta}^{\texttt{TPR}}_b$, then $\hat{\theta}^{\texttt{TPR}}_g(\beta)> {\theta}^{\texttt{TPR}}_g$ for both groups and any function $\beta(x)$. Further, the \texttt{TPR} constraint is violated at the new thresholds. 

~~ (iii) If $\hat{F}^0_b(x)<F^0_b(x), \forall x\geq {\theta}^{\texttt{FPR}}_b$, then $\hat{\theta}^{\texttt{FPR}}_g(\beta)> {\theta}^{\texttt{FPR}}_g$ for both groups and any function $\beta(x)$. Further, the \texttt{FPR} constraint is violated at the new thresholds. 

~~ (iv) If $\hat{F}_b(x)<F_b(x), \forall x\geq {\theta}^{\texttt{DP}}_b$, then $\hat{\theta}^{\texttt{DP}}_g(\beta)> {\theta}^{\texttt{DP}}_g$ for both groups and any function $\beta(x)$. Further, the \texttt{DP} constraint is violated at the new thresholds. 

~~ (v) There exist problem instances in which $\hat{\theta}^{\texttt{f}}_b(\beta)<{\theta}^{\texttt{f}}_b$ for any of the three constraints. 
\end{proposition}

We provide a visualization of this proposition in Figure~\ref{fig:synthetic-feature-util-contour} in Appendix~\ref{app:experiments}. This proposition shows that unless the feature measurement errors only affect one label (as in part (i)), the considered fairness constraints will in general not remain robust against feature measurement errors. The conditions in parts (ii)-(iv) require that a CDF in the biased distribution first-order stochastically dominates that of the unbiased distribution. This holds if, e.g., the corresponding features (qualified agents, unqualified agents, or all agents, respectively) are underestimated, $\hat{x}=x-\epsilon$ for some $\epsilon\geq 0$. We also note that in contrast to Proposition~\ref{prop:f-bias-gamma}, the decision threshold can in fact \emph{decrease} when biases are introduced; we illustrate this in our experiments in Section~\ref{sec:experiments}. 

Similar to Proposition~\ref{prop:f-bias-gamma-sensitivity}, we can also characterize the sensitivity of each constraint to bias rates, and investigate the impacts of other problem parameters on these sensitivities. We present this in detail in Proposition~\ref{prop:f-bias-feature-sensitivity} in Appendix~\ref{app:prop-f-bias-feature-sensitivity}.
\section{Numerical Experiments}\label{sec:experiments}
We now provide numerical support for our analytical results, and additional insights into the robustness of different fairness measures, through experiments on both real-world and synthetic datasets. Details about the datasets, experimental setup, and additional experiments, are given in Appendix~\ref{app:experiments}.

\begin{figure*}[h!]
\centering
\includegraphics[width=0.8\textwidth]{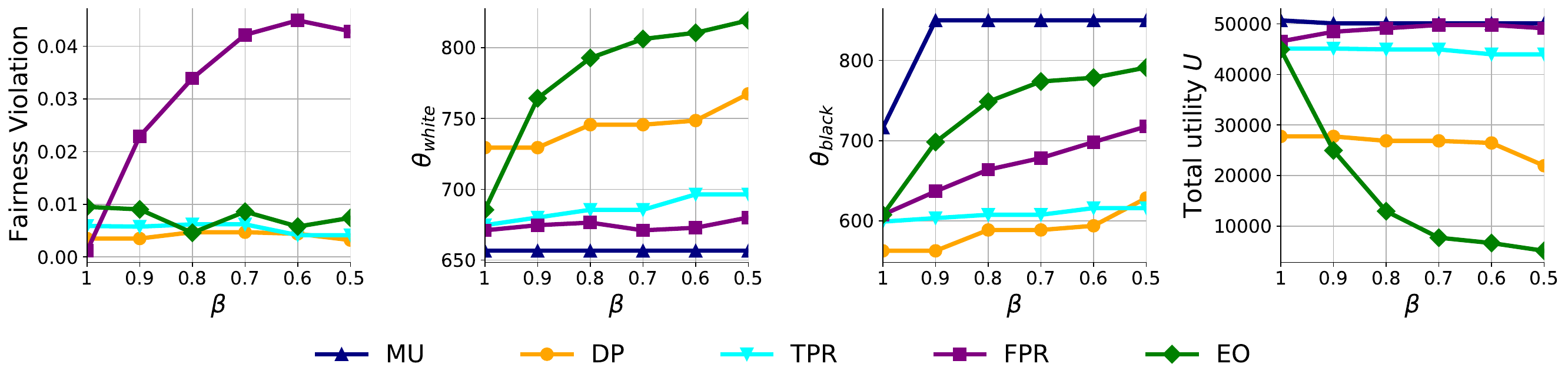}
\caption{Experiments on qualification assessment (labeling) errors on the disadvantaged group in the FICO credit score dataset.}
\label{fig:FICO-experiments}
\end{figure*}

\subsection{FICO Credit Score Dataset}
We begin with numerical experiments on the FICO dataset preprocessed by \cite{hardt2016equality}. The FICO credit scores ranging from 300 to 850 correspond to the one-dimensional feature $x$ in our model, and race is the sensitive feature $g$ (we focus on the white and black groups). The data provides repay probabilities for each score and group, which corresponds to our qualification profile $\gamma_g(x)$. 
We take this data to be the unbiased ground truth. (We discuss some implications of this assumption in Appendix~\ref{app:discussion}.) To induce labeling biases in the data, we drop the repay probabilities of the black group to model the underestimation of their qualification profiles, and generate training data on this group based on this biased profile. We use $\beta$ to parametrize the bias (with $\hat{\gamma}_b(x)=\beta\gamma_b(x)$).  Decision rules will be found on the biased data and applied to the unbiased data.

\textbf{Violation of fairness constraints.} The left-most panel in Figure~\ref{fig:FICO-experiments} illustrates the fairness violation under each fairness constraint (measured as $|\mathcal{C}_a^{\texttt{f}}(\hat{\theta}^\texttt{f}_a) - \mathcal{C}_b^{\texttt{f}}(\hat{\theta}^\texttt{f}_b)|$) as qualification assessment biases on the disadvantaged group increase. These observations are consistent with Proposition~\ref{prop:f-bias-gamma}. %
In particular, \texttt{DP} and \texttt{TPR} are both robust to these biases in terms of achieving their notions of fairness, while \texttt{FPR} has an increasing trend in fairness violation. This means that the set of possible decision rules of \texttt{FPR} changes when bias is imposed. Note that though a violation of 0.04 may not be severe, it is more than 300\% higher than a violation below 0.01 that \texttt{FPR} can achieve when the data is unbiased. We will also observe more significant violations of \texttt{FPR} on other datasets. 
Finally, from Figure~\ref{fig:FICO-experiments}, it may seem that \texttt{EO} also remains relatively robust to bias. This observation is not in general true (as shown in our experiments on other datasets in Section~\ref{sec:adult-german}); however, it can be explained for the FICO dataset by noting how \texttt{EO}'s feasible pairs of decision rules change due to data bias (similar to Figure~\ref{fig:synthetic-label-util-contour-prop-1}) and how the problem setup can influence the results. We provide additional discussion in Appendix~\ref{app:extra-fico}. 

\textbf{Changes in decision thresholds and firm's utility.} In Figure~\ref{fig:FICO-experiments}, we also display the threshold change of each group. In line with Proposition~\ref{prop:f-bias-gamma}, thresholds for both groups increase as the bias level increases. Notably, the maximum utility (fairness unconstrained) decision rule would have led the firm to fully exclude the black group even at relatively low bias rates; all fairness-constrained thresholds prevent this from happening. In addition, the threshold's increase under \texttt{TPR} is less drastic than \texttt{DP} and \texttt{EO} (and consistent with Corollary~\ref{cor:DP-sensitive}). 

Finally, we note the changes in the firm's utility. As the decision thresholds increase due to biases, the net utility from the white/black groups ($U_{white}, U_{black}$) increases/decreases. 
Overall, due to the fact that the white group is the majority in this data, an increase in the threshold $\theta_{white}$ 
will lead to a greater loss in total utility of the firm (as is the case in \texttt{DP/TPR/EO} seen in Figure~\ref{fig:FICO-experiments}). That said, the total utility may increase under \texttt{FPR}, as pointed out in Proposition~\ref{prop:f-bias-gamma} and observed in Figure~\ref{fig:FICO-experiments}, since there is a gain from $U_{black}$ is larger than the loss from $U_{white}$. This increase may even make the \texttt{FPR}-constrained classifier attain higher utility than \texttt{MU} when training data is highly biased.

\subsection{Adult Dataset and German Credit Dataset}\label{sec:adult-german}
We next conduct experiments on two additional benchmark datasets: the Adult dataset and the German credit dataset \cite{Dua:2019}. 
In both these datasets, instead of maximizing utility, the objective is classification accuracy. We first train a logistic regression classifier on the training set using scikit-learn \cite{pedregosa2011scikit} with default parameter settings as the base model. \yl{The logistic regression output in the range 0 to 1 can be interpreted as the score of being qualified as in the FICO dataset.} Then, we obtain the fair classifier by applying the exponentiated gradient reduction \cite{agarwal2018reductions} using Fairlearn \cite{bird2020fairlearn}. \yl{Although the exponentiated gradient reduction produces a randomized classifier, it can be viewed as an abstract threshold: given a randomized classifier, the expected rates (e.g., selection rates, true positive rates) for both groups are deterministic. We thus find that our claims on how labeling biases impact \texttt{DP/TPR/FPR} still hold.}

We introduce qualification assessment biases by flipping the qualification states $y$ of the (1) qualified agents from the disadvantaged group (i.e., female in Adult and age below 30 in German), (2) unqualified agents from the advantaged group (i.e., male in Adult and age above 30 in German).\footnote{Results from concurrently flipping the labels of qualified agents from the disadvantaged group and unqualified agents from the advantaged group are shown in Appendix~\ref{app:extra-adult-german}.}

\begin{figure*}[t]
    \centering
    \begin{subfigure}[t]{0.38\textwidth}
        \centering
        \includegraphics[width=\textwidth]{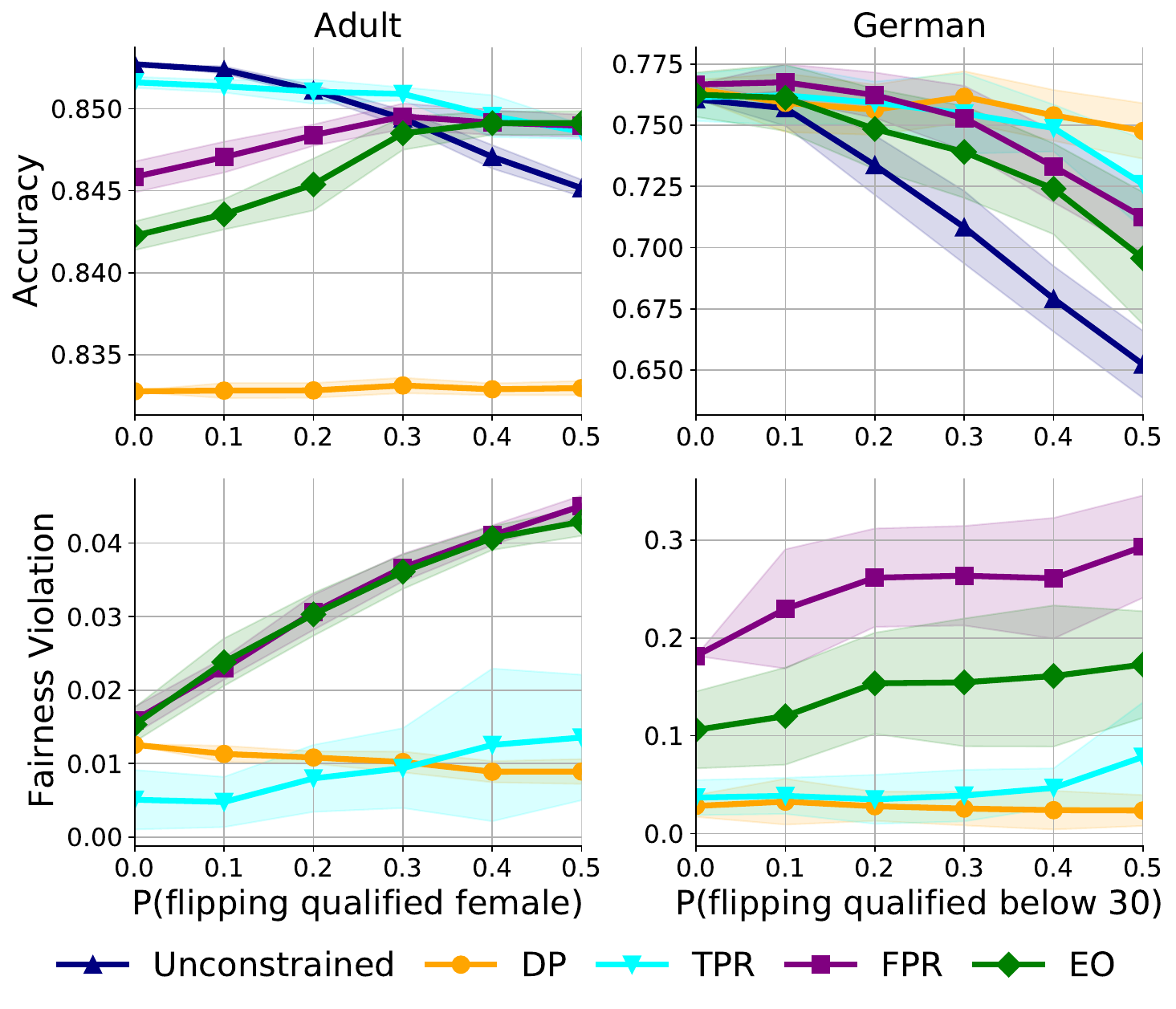}
        \caption{Flip the qualified disadvantaged group.}
    \end{subfigure}
    \hspace{0.4in}
    \begin{subfigure}[t]{0.38\textwidth}
        \centering
        \includegraphics[width=\textwidth]{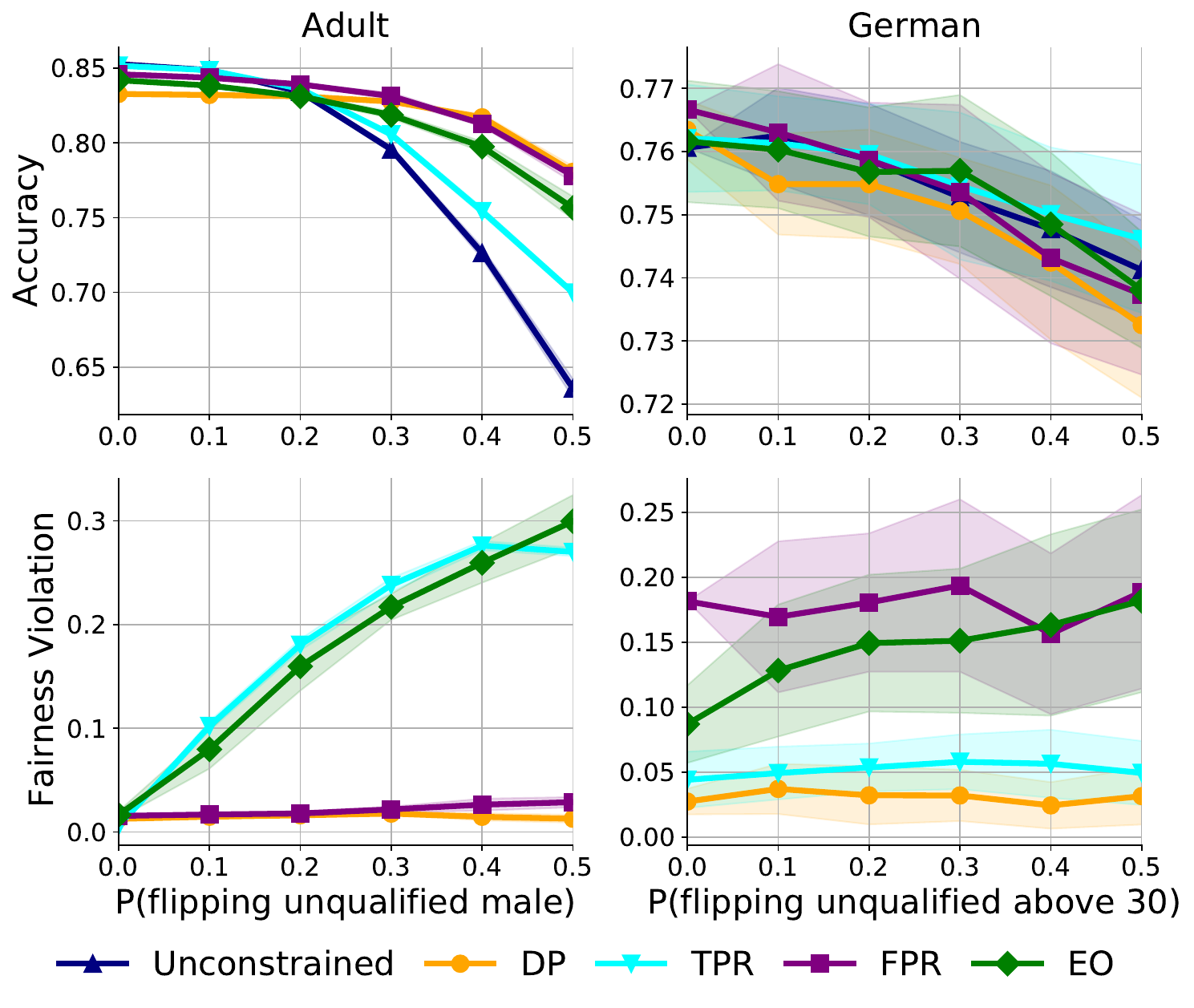}
        \caption{Flip the unqualified advantaged group.}
    \end{subfigure}
    \caption{Accuracy and fairness violation on Adult and German datasets. The results are averaged across 10 runs for Adult and 30 runs for German; shaded areas indicate plus/minus standard deviation.}
    \label{fig:adult-german-acc-fairness}
\end{figure*}

\begin{figure*}[h]
\centering
\includegraphics[width=0.75\textwidth]{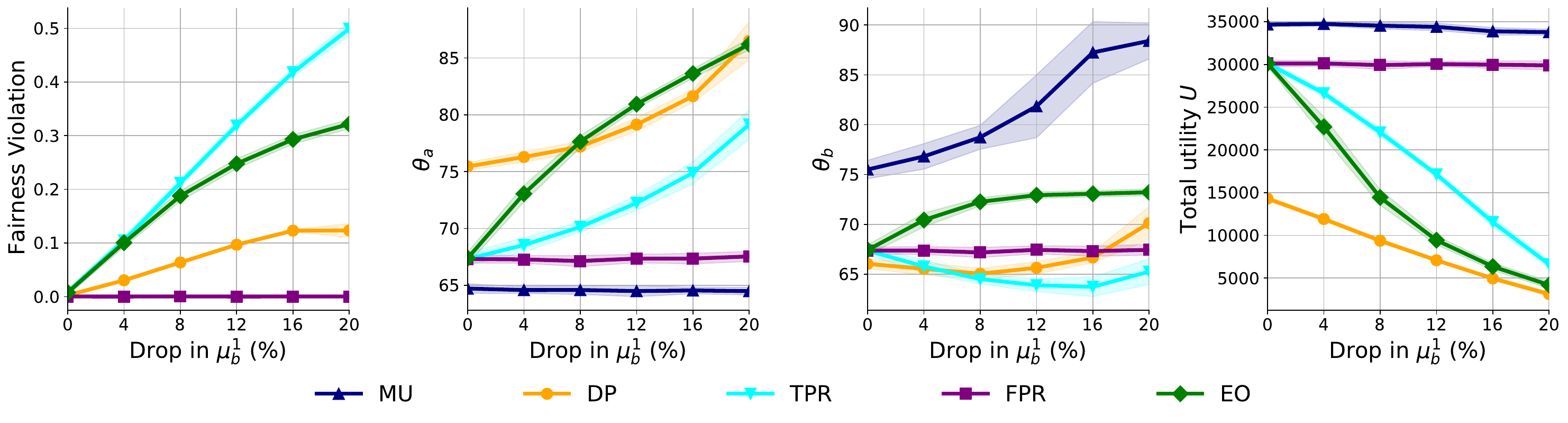}
\caption{Fairness violation, thresholds, and utility under different constraints from synthetic simulation on measurement errors. The results are from 20 runs.}
\label{fig:synthetic-feature_b1-fairness-thresh-util}
\end{figure*}

The results are presented in Figure~\ref{fig:adult-german-acc-fairness}. To quantify the trend in fairness violation, we fit a linear regression model to the fairness violation and present all model weights in Table~\ref{t:lr-weights-table} in Appendix~\ref{app:extra-adult-german}. In all three cases, the robustness of each constraint in terms of achieving fairness matches our findings in Propositions~\ref{prop:f-bias-gamma} and \ref{prop:f-bias-gamma-a}. One exception, however, is that \texttt{TPR} remains robust when we flip the labels of the unqualified advantaged group in the German dataset. This is primarily because, while flipping the unqualified individuals in the training set will in general make the classifier accept more of these individuals, the flip will have a minor effect on the \texttt{TPR} violation because 1) there is a limited number of unqualified individuals with age above 30 (15.2\% of the dataset) compared to qualified individuals with age above 30 (43.7\% of the dataset), and 2) there is little room for the true positive rate values on the test set to increase since the values for both groups start close to 1 (0.923 and 0.845) with a small difference (0.078) (see Figure~\ref{fig:german-advantaged-0-TPR-vals} in Appendix~\ref{app:extra-adult-german}).

Interestingly, we also observe that fairness-constrained accuracy can be \emph{higher} than the utility-maximizing choices when training data is biased, as seen in the top row accuracy plots in Figure~\ref{fig:adult-german-acc-fairness}. This can be interpreted as fairness constraints having a regularization effect: since the constraints prevent the classifier from overfitting the biased training set to some extent, the test accuracy with fairness constraints imposed can be higher than that of the unconstrained case.

\subsection{Impacts of Feature Measurement Errors}\label{sec:experiment-synthetic}
Lastly, we conduct experiments on a synthetic dataset inspired by the FICO credit score data (details in Appendix~\ref{app:synthetic-label-bias}). To bias the feature measurements, we drop the estimate $\hat{\mu}^1_b$ of the mean of the qualified agents from group $b$ relative to the true value ${\mu}^1_b$. As a result, $\hat{f}^1_b$ will be biased relative to its true value, while $\hat{f}^0_b$ will remain unchanged. As shown in Figure~\ref{fig:synthetic-feature_b1-fairness-thresh-util}, and consistent with Proposition~\ref{prop:f-bias-feature}, under these choices, $\texttt{FPR}$ will remain unaffected, while \texttt{DP/TPR} will no longer be satisfied.

Finally, Figure~\ref{fig:synthetic-feature_b1-fairness-thresh-util} also highlights the changes in the decision thresholds and firm's utility under this type of bias. As noted in Proposition~\ref{prop:f-bias-feature}, now the thresholds on the disadvantaged group can \emph{decrease} compared to the unbiased case at low bias rates. Notably, we also observe that the firm's overall utility is lower under \texttt{DP} (similar to the labeling bias case), but that \texttt{TPR} is more sensitive to bias levels than \texttt{DP} (unlike the labeling bias case). This points to the fact that the choice of a robust fairness constraint has to be made subject to the type of data bias that the decision maker foresees. 
\section{Conclusion}\label{sec:conclusion}

We investigated the robustness of different fairness criteria when an algorithm is trained on statistically biased data. We provided both analytical results and numerical experiments based on three real-world datasets (FICO, Adult, and German credit score). We find that different constraints exhibit different sensitivity to labeling biases and feature measurement errors. In particular, we identified fairness constraints that can remain robust against certain forms of statistical biases (e.g., Demographic Parity and Equality of Opportunity given labeling biases on the disadvantaged group), as well as instances in which the adoption of a fair algorithm can increase the firm's expected utility when training data is biased, providing additional motivation for adopting fair machine learning algorithms. Our findings present an additional guideline to go along with normative considerations, for choosing among existing fairness criteria when available datasets are biased. 
We provide additional discussion about other implications of our findings, limitations, and future directions, in Appendix~\ref{app:discussion}.

\section*{Acknowledgments}
The authors are grateful for support from the NSF program on Fairness in AI in collaboration with Amazon under Award No. IIS-2040800, and Cisco Research. Any opinions, findings, and conclusions or recommendations expressed in this material are those of the authors and do not necessarily reflect the views of the NSF, Amazon, or Cisco.


\bibliography{aaai23}

\clearpage
\onecolumn
\appendix
\section*{Appendix}
\section{Discussion, Limitations, and Future Work}\label{app:discussion}

\paragraph{Additional motivation for adopting fair algorithms when data is biased:} We begin by noting that despite their differing robustness, we find that any fairness-constrained algorithm can couple the decisions between two demographic groups in a way that benefits the disadvantaged group suffering from data biases. This was highlighted in our numerical experiments on the FICO dataset, where a utility-maximizing firm would have fully excluded the black group even at relatively low bias rates, while all fairness-constrained decisions prevented this from happening. In addition, we showed that there exists settings in which a fairness-constrained classifier can outperform utility-maximizing classifiers from a firm's perspective when data is biased. Together, our results provide additional motivation to adopt fair machine learning algorithms: they can simultaneously benefit disadvantaged demographic groups and increase (utility-maximizing) firms' expected profit.  

\paragraph{Implications for fairness criteria selection/design:} Our work in this paper takes existing fairness constraints and provides a framework for assessing their robustness to different forms of statistical data bias. In practice, normative or legal constraints, rather than robustness against data bias, may be the main driving forces behind the selection of a given fairness metric. In that light, our work has the following implications. First, if a fairness constraint is already identified, then our methodology could provide insights into its robustness, and guide data debiasing efforts. For instance, if a decision maker selects the \texttt{TPR/Equality of Opportunity} fairness constraint as their criteria, our work shows that the constraint can continue providing its desired notion of fairness even if there is selective labeling on the disadvantaged user population. It also implies that to prevent the firm from losing fairness guarantees and/or utility due to selection of the incorrect set of \texttt{TPR}-satisfying thresholds when feature measurement errors are probable, the decision maker can focus on ensuring that estimates $\hat{f}^1_g(x)$, the feature distribution of qualified individuals in each group, are statistically unbiased or debiased through data collection efforts, as this is the statistic used by the \texttt{TPR} constraint in assessing and imposing its notion of fairness. 

Alternatively, if the decision maker is designing a new criterion to satisfy a desired (normative) notion of fairness, the new criterion will rely on certain statistics about the user population (e.g., feature-label distributions, or representation rates). Our proposed framework shows that the satisfiability of this constraint rests largely on the accuracy of said statistics, and can therefore guide the need for re-assessing the constraint design, or simultaneously guide the decision maker's data collection efforts. 

\paragraph{On one-dimensional features and extension to multi-dimensional features:} Our analytical work in this paper is based on one-dimensional feature data and threshold classifiers. Our motivation for this choice is that it is possible to reduce multi-dimensional features to one-dimensional ``scores'',  and focus on threshold classifier accordingly. This might not necessarily be a restrictive choice: its optimality has been established in \citet{corbett2017algorithmic} [Thm 3.2] as long as a multi-dimensional feature $X$ can be mapped to a properly defined scalar. The recent advances in deep learning also align with this possibility: one can take the last layer output from a deep neural network and use it as the single dimensional representation. In addition, our simulation results show that the same insights obtained from our analytical studies can be observed for multi-dimensional data, as well as when using randomized classifiers (see also the discussion in the first paragraph of Section 4.2).

\paragraph{Limitations and future work} We next discuss some limitations of our findings, and potential directions of future inquiries. 
First, our analysis has focused on a subset of demographic (group) fairness criteria; whether other fairness criteria (including individual fairness criteria) will exhibit different sensitivity to statistical data biases remains to be studied. We believe our proposed framework and analysis approach (e.g., convex optimization techniques, the implicit function theorem) provide a starting point for these investigations. More importantly, we have not focused on the normative value of different fairness criteria, but rather taken these as potential desiderata. Quantifying unfairness and discrimination remains an important open question, and will in general be shaped by cultural, legal, and political perspectives.  We hope that our proposed approach to quantifying and assessing the impacts of statistical biases on the efficacy of a desired (normative) notion of fairness can contribute to these conversations, and also guide the design of new criteria, as discussed earlier. 

As noted in our experiments, we have taken existing real-world datasets as ground truth (i.e., unbiased). This assumption has been inevitable for us due to our lack of access to ground truth data. A main motivation of our work, however, is that existing datasets are likely to suffer from various forms of statistical biases (including not only labeling biases and measurement errors but also disparate representation, changes in qualification rates over time, etc.). If the existing bias in these datasets is of the same type we have considered (e.g., labeling biases on qualified, disadvantaged agents), 
some of our findings continue to be supported as they can be viewed as additional bias added on top of existing ones (e.g., \texttt{DP/TPR} are robust in face of a range of label flipping probabilities, or thresholds monotonically increase as labeling bias increases). However, experiments on better benchmarks are indeed desirable. 

Investigating other forms of dataset biases, as well as the concurrent presence of multiple forms of bias, also remains an open question. Notably, a third form of statistical bias may be due to {long-term changes in qualification rates} (reflected as errors in $\alpha_g$ in our model). These may occur if a firm does not account for improvement efforts made by individuals over time, and has been studied as population dynamics models studied in \cite{zhang2019group,zhang2020fair}. Robustness of fairness criteria against these forms of biases, and whether they consequently support improvement efforts, is an interesting extension. More broadly, we have looked at the interplay between data biases and algorithmic fairness criteria in a static model. Further investigation of the feedback loops between these two remains an important open challenge.

\section{Additional Related Work}\label{app:related}

A variety of fairness criteria have been proposed with the goal of formalizing desired notions of algorithmic fairness (\citealt{mehrabi2021survey} and \citealt{barocas2017fairness} provide excellent overviews). Our focus in this paper is on four of these (group) fairness criteria: {Demographic Parity (DP)}, {True/False Positive Rate Parity (TPR/FPR)}, and {Equalized Odds (EO)}. This allows us to consider representative fairness criteria from two general categories~\citep{barocas2017fairness}: \emph{independence} (which requires that decisions be statistically independent of group membership, as is the case in \texttt{DP}) and \emph{separation} (which requires that decisions be statistically independent of group membership when conditioned on qualification, as is the case in \texttt{TPR/FPR/EO}). We also note that fairness criteria from different categories are in general incompatible with each other (see \citealt{barocas2017fairness,chouldechova2017fair,kleinberg2016inherent}); this points to an inherent trade-off between these different notions of fairness. Our work on assessing the robustness of these criteria to data biases introduces an additional metric against which to compare them. 

A variety of approaches have been proposed for achieving a given notion of fairness, and generally fall into three categories: (1) \emph{pre-processing}, which modifies the training dataset through feature selection or re-weighing techniques (e.g., \citealt{kamiran2012data,jiang2020identifying}), (2) \emph{in-processing}, which imposes the fairness criteria as a constraint at training time (e.g., \citealt{agarwal2018reductions,zafar2019fairness}), and (3) \emph{post-processing}, which adjusts the output of the algorithm based on the sensitive attribute (e.g., \citealt{hardt2016equality}). We consider the in-processing approach and formulate the design of a fair classifier as a constrained optimization problem.

Our work is also related to the recent work on fair classification in the presence of noisy labels by \citet{wang2021fair}. The main difference between our setting and this line of work is that we assume \emph{specific} forms of label noise, while this work assumes a noise generating process that leads to noise over both groups and both labels. In particular, the main scenario of interest for us in Section \ref{sec:gamma-bias} is that of selective labeling on the disadvantaged group. This choice, as well as our analysis in the subsequent sections on other forms of data bias, allows us to provide a sharper focus on the specific impact of each form of label bias on each type of fairness constraint. In addition, \citep{blum2020recovering,wang2021fair} use empirical risk minimization frameworks in studying fair learning under biased labels; in contrast, our analytical framework and modeling of data bias as inaccuracies in the problem primitives ($\hat{\gamma}(x), \hat{f}^y_x, \ldots$) is novel and different from prior work.

\section{Summary of Notation}\label{app:notation}
Our notation is summarized in Table~\ref{t:notation}. 

\begin{table*}[h]
    \caption{Summary of notation.}
    \label{t:notation}
    \centering
    \begin{tabular}{c p{6.4cm}}
    \hline
        Notation & Description \\
    \hline
        $g$ & Demographic groups, $g\in\{a,b\}$\\ 
    \hline
        $n_g$ & Fraction of population from group $g$\\
    \hline
        $x$ & Observable feature, $x\in\mathbb{R}$\\
    \hline
       $y$ & True qualification state, $y\in\{0,1\}$\\
    \hline
    $\alpha_g$ & Qualification rate of group $g$\\
    \hline
    $f^y_g(x)$ & Feature distribution of agents with label $y$ from group $g$\\
    \hline
    $f_g(x)$ & Feature distribution of all agents from group $g$\\
    \hline
    $\gamma_g(x)$ & Qualification profile of group $g$; probability that agent with feature $x$ from group $g$ is qualified\\
    \hline
    $d$ & Firm's accept/reject decision, $d\in\{0,1\}$\\
    \hline
    $\theta_g$ & Firm's threshold policy on group $g$\\
    \hline
    $u_+/u_-$ & Firm's benefit/loss from accepting qualified/unqualified agents\\
    \hline
    $U(\theta_a, \theta_b)$ & Firm's expected payoff given policies $\{\theta_a, \theta_b\}$\\
    \hline
    $C_g^\texttt{f}(\theta_g)$ & Fairness measure on group $g$\\ 
    \hline
    $\theta_g^{\texttt{MU}}$ & Firm's optimal threshold on group $g$ for maximum utility (fairness unconstrained)\\
    \hline
    $\theta_g^{\texttt{f}}$ & Firm's optimal threshold on group $g$ under fairness constraint \texttt{f}\\
    \hline
    \end{tabular}
\end{table*}

\section{Additional Experiment Details and Results}\label{app:experiments} 

\subsection{Parameter setup for the synthetic dataset in Figure~\ref{fig:synthetic-label-util-contour-prop-1} and Section~\ref{sec:experiment-synthetic}}\label{app:synthetic-label-bias}
The group ratios and qualification rates are set as follows: $n_a$ = 0.8, $n_b$ = 0.2, $\alpha_a$ = 0.8, and $\alpha_b$ = 0.3. The feature $x$ of each agent is sampled from a Gaussian distribution with $\mu$ = 70 and $\sigma$ = 10 if qualified, and from a Gaussian distribution with $\mu$ = 50 and $\sigma$ = 10 if unqualified. A total of 100000 examples are sampled. $\frac{u_-}{u_+}$ is set to 10.

\subsection{FICO credit score experiments}\label{app:extra-fico}

\paragraph{Additional experiment setup details.} We set $\frac{u_-}{u_+}$ to 10; this is selected as losses due to defaults on loans are typically higher for a bank than benefits/interests from on-time payments. We focus on the black and white groups in our discussion. Within the FICO dataset, $n_{black}$ and $n_{white}$ are 0.12 and 0.88, and $\alpha_{black}$ and $\alpha_{white}$ are 0.34 and 0.76. We analyze the effects of bias with and without fairness constraints. Unless otherwise specified, we use soft constraints $|\mathcal{C}_a^{\texttt{f}}(\theta_a) - \mathcal{C}_b^{\texttt{f}}(\theta_b)| \leq 0.01$ in our experiments. We measure the fairness violation under each constraint based on its fairness definition with respect to the level of bias we impose on the data. We will also highlight the change in the decision thresholds, selection rates of each group, and firm's utility (both group-wise and total). Note that the FICO dataset comes with the repay probability at each score, so there is no randomness in the experiments when we induce biases.

\paragraph{Experiment results and discussion.} Figure~\ref{fig:fico-thresh-select-util} provides additional details on the changes in selection rate, and the change in the utility from each group, in this experiment.

\begin{figure*}[h]
\centering
\includegraphics[width=0.9\linewidth]{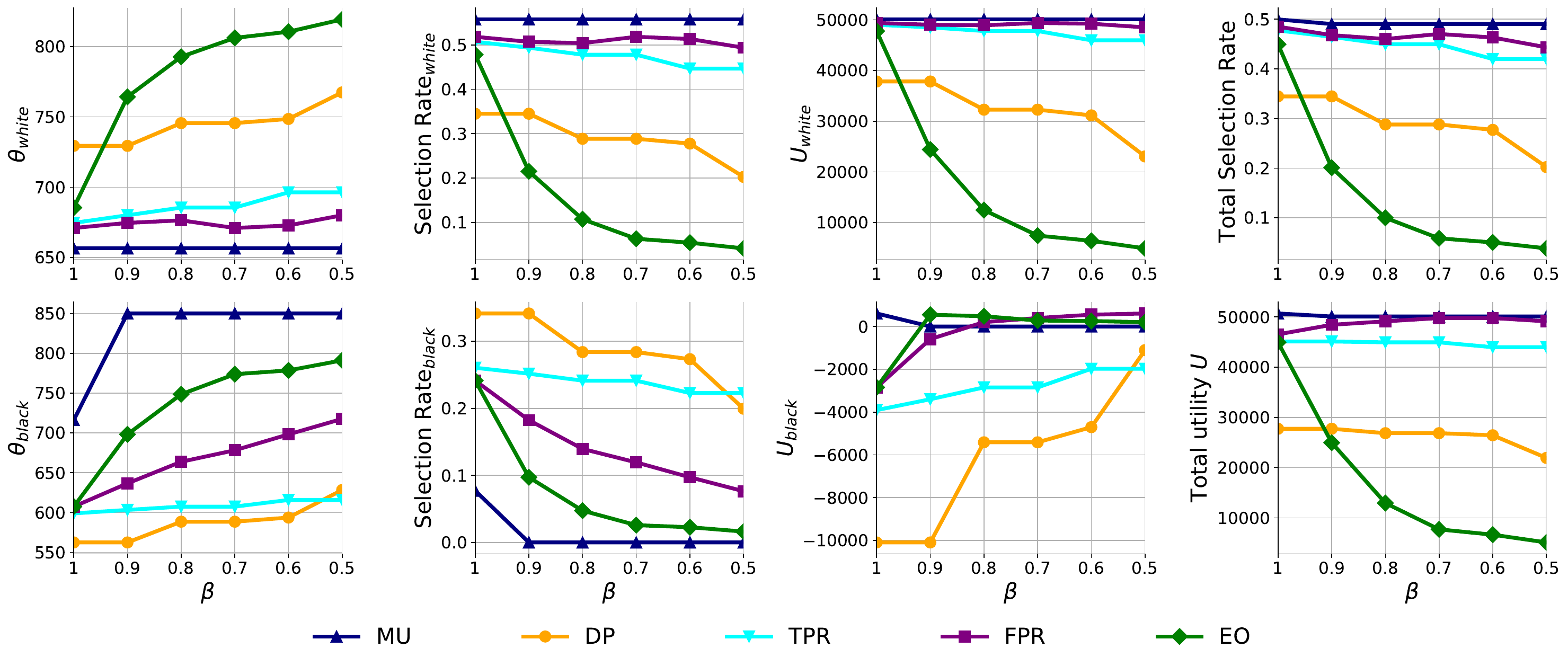}
\caption{Thresholds, selection rates, and utility from FICO experiments.}
\label{fig:fico-thresh-select-util}
\end{figure*}

\begin{figure}[h]
    \centering
    \includegraphics[width=0.45\textwidth]{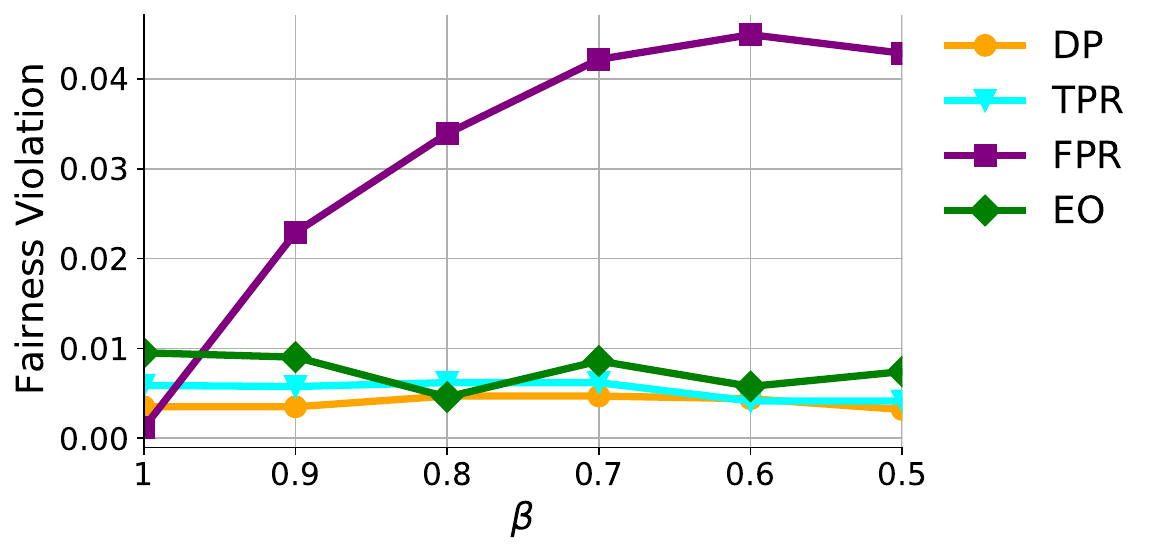}
    \caption{Fairness violation on the FICO dataset under different fairness constraints.}
    \label{fig:fico-fairness}
\end{figure}

\paragraph{Additional intuition for the observed fairness violations.}
From Figure~\ref{fig:fico-fairness}, \texttt{DP} and \texttt{TPR} are robust to underestimation of qualification profiles in terms of achieving their notions of fairness. Intuitively, as noted earlier, this can be explained as follows. Based on the definition of \texttt{DP}, the constraint tries to equalize the selection rates of two groups regardless of the true qualification state of the agents. On the other hand, flipping labels from 1 to 0 does not alter the number of agents from group $b$ at each score. Thus, each pair of decision rules that could achieve \texttt{DP} on the original unbiased data will still be able to achieve \texttt{DP} on the biased data. Conversely, the decision rules found on the biased data satisfying \texttt{DP} remain fair when applied back to the unbiased data.

For \texttt{TPR} on the other hand, recall the definition, which equalizes the ratio of the number of accepted qualified agents to the total number of qualified agents. Underestimating the qualification profiles would be equivalent to dropping the number of qualified agents at each score. When calculating the true positive rate, we decrease the two quantities in the ratio by the same fraction. Consequently, similar to \texttt{DP}, the set of decision rules satisfying \texttt{TPR} remains the same after the bias is imposed. 

Figure~\ref{fig:fico-fairness} also shows that \texttt{FPR} has an increasing trend in fairness violation. This means that the set of possible decision rules of \texttt{FPR} changes when bias is imposed. To see why, note that false positive rate computes the ratio of the number of accepted unqualified agents to the total number of unqualified agents. Dropping the qualification profiles increases the two quantities in the ratio at different rates, causing the false positive rates to change at given thresholds.

Finally, from Figure~\ref{fig:fico-fairness}, it may seem that \texttt{EO} also remains relatively robust to bias. This observation is not in general true (as shown in our experiments on other datasets in Section~\ref{sec:adult-german}); however, it can be explained for the FICO dataset as follows. As \texttt{EO} requires satisfying both \texttt{TPR} and \texttt{FPR}, the set of decision rules achieving \texttt{EO} will be at most the intersection of those of \texttt{TPR} and \texttt{FPR}. As noted above, the set of feasible thresholds are the same on biased and unbiased data when requiring \texttt{TPR}, while they differ on \texttt{FPR}. As a result, the set of possible decision rules for \texttt{EO} will change with bias (here, it becomes smaller). That said, the pairs that remain feasible, for this dataset, are those that are dominated by the requirement to satisfy \texttt{TPR}, which are themselves robust to biases. 
(The same effect is also highlighted in Figure~\ref{fig:synthetic-label-util-contour-prop-1}, which illustrates the changes in the set of feasible thresholds in experiments on synthetic data.) Moreover, we enforce that the constraints must be satisfied (i.e., $|\mathcal{C}_a^{\texttt{f}}(\theta_a) - \mathcal{C}_b^{\texttt{f}}(\theta_b)| \leq 0.01$) during training, and the original unbiased data is used for testing. In contrast, if the constraints are imposed by introducing an extra penalty term to the objective function, the model may retain high accuracy/utility by not perfectly satisfying the fairness constraints depending on the exact formulation. Besides, the constraints may not be satisfied even when there is no bias due to the gap between training and test data. Nevertheless, our focus is the trend of how the performance of each constraint reacts to the increasing level of bias (as we discuss further in Section~\ref{sec:adult-german}).

\subsection{Adult and German credit score datasets}\label{app:extra-adult-german}

\paragraph{Additional experiment setup details.} (i) the Adult dataset: the goal is to predict if an individual would make more than 50k per year. We choose gender as the sensitive feature. We followed \citep[Chapter 2]{barocas2017fairness} to preprocess the data. 

(ii) the German credit dataset: the goal is to predict if an individual has a good or bad credit risk. The sensitive feature is age. Following \citep{jiang2020identifying}, we binarize age using a cutoff threshold of 30 when deciding group membership. 67\% of the data is randomly selected as the training set, and the rest is used as the test set.

In both experiments, the maximum number of iterations is set such that early stopping is avoided.

\paragraph{Flipping labels on both advantaged and disadvantaged individuals.} The results are illustrated in Figure~\ref{fig:both-flippped}, and are again consistent with our analytical arguments. In particular, \texttt{DP} is robust against both forms of label biases. 

\begin{figure}[h]
        \centering
        \includegraphics[width=0.45\textwidth]{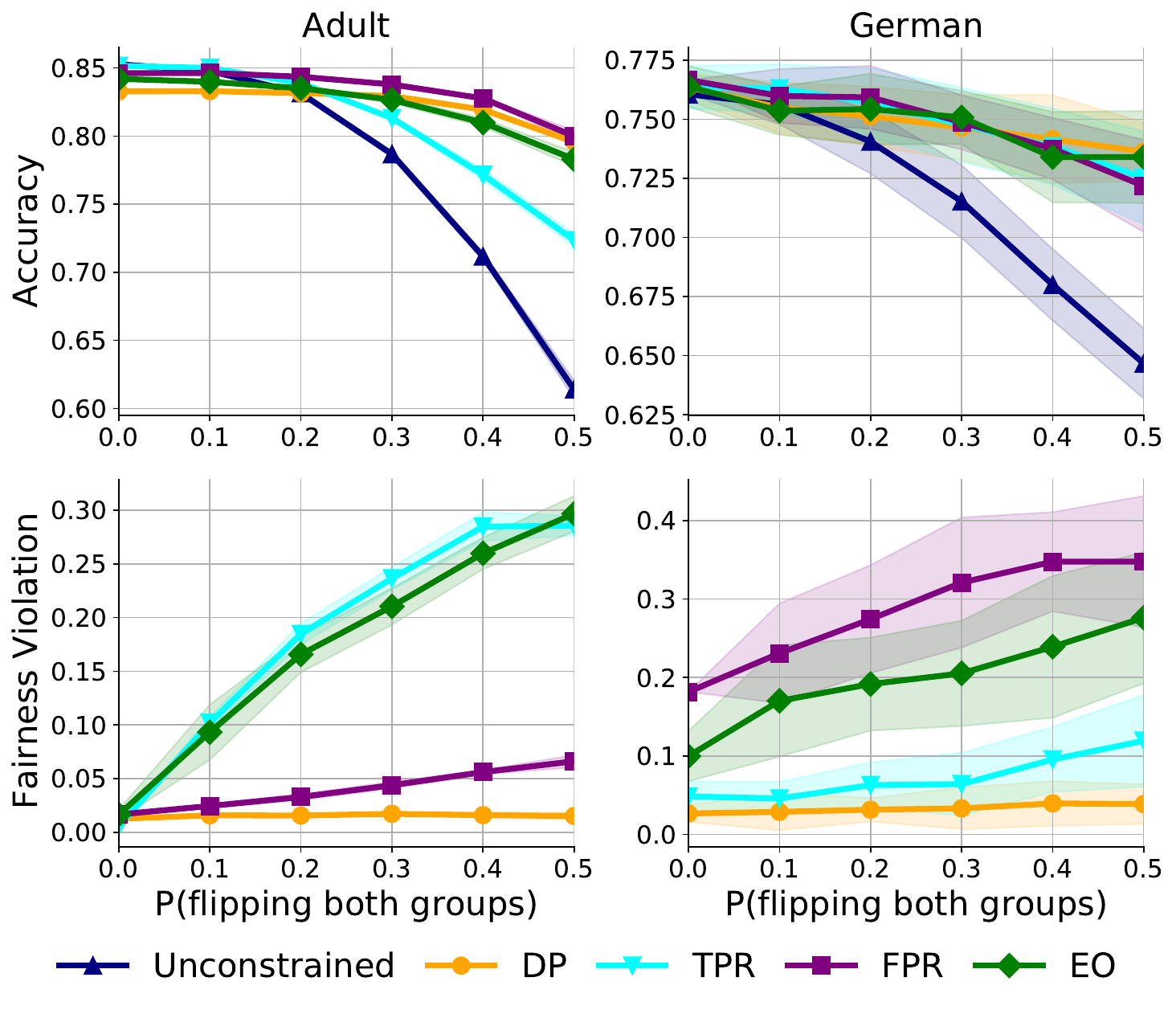}
        \caption{Flip the labels of both the qualified disadvantaged group and unqualified advantaged group.}
        \label{fig:both-flippped}
\end{figure}

\paragraph{Linear regression weights showing the relationship between fairness violation and bias level.} These are shown in Table~\ref{t:lr-weights-table} for Figures~\ref{fig:adult-german-acc-fairness} and \ref{fig:both-flippped}. 

\begin{table*}[h]
  \caption{Linear regression weights showing the relationship between fairness violation and bias level.}
  \label{t:lr-weights-table}
  \centering
  \begin{tabular}{llccc}
    \toprule
    \multirow{2}{*}{Dataset}    & \multirow{2}{*}{Constraint} & \multicolumn{3}{c}{Weight} \\
    \cmidrule(r){3-5}       &   & Flip qualified disadvantaged & Flip unqualified advantaged & Flip both \\
    \midrule
    \multirow{4}{*}{Adult}  & \texttt{DP}   & \textbf{-0.007}   & \textbf{0.001}    & \textbf{0.004} \\
                            & \texttt{TPR}  & \textbf{0.019}    & 0.544             & 0.572 \\
                            & \texttt{FPR}  & 0.059             & \textbf{0.028}    & 0.101 \\
                            & \texttt{EO}   & 0.056             & 0.575             & 0.556 \\
    \midrule
    \multirow{4}{*}{German} & \texttt{DP}   & \textbf{-0.015}   & \textbf{-0.005}   & \textbf{0.027} \\
                            & \texttt{TPR}  & \textbf{0.068}    & \textbf{0.007}    & 0.144 \\
                            & \texttt{FPR}  & 0.187             & \textbf{0.003}    & 0.350 \\
                            & \texttt{EO}   & 0.131             & 0.166             & 0.315 \\
    \bottomrule
  \end{tabular}
\end{table*}
\begin{figure}[h]
    \centering
    \includegraphics[width=0.45\textwidth]{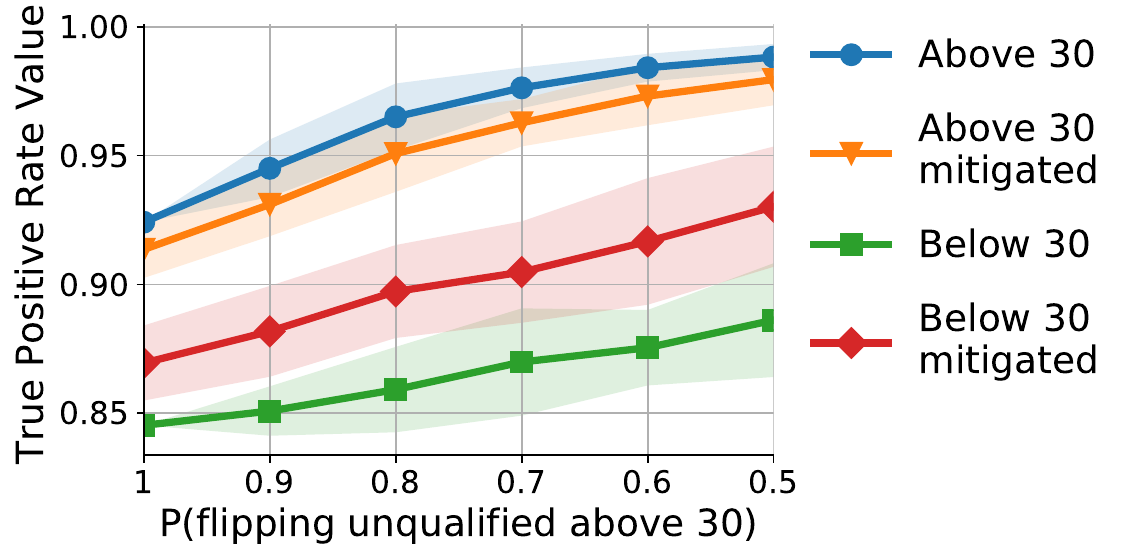}
    \caption{Group-wise true positive rate values for German under \texttt{TPR}.}
    \label{fig:german-advantaged-0-TPR-vals}
\end{figure}

\subsection{Synthetic dataset with measurement errors}\label{app:synthetic-feature-bias}
The parameter settings are identical to those in Appendix \ref{app:synthetic-label-bias}. Figure~\ref{fig:synthetic-feature_b1-thresh-select-util.} provides additional details on the changes in selection rate, and the change in the utility from each group, in this experiment.

\begin{figure*}[h]
\centering
\includegraphics[width=0.9\textwidth]{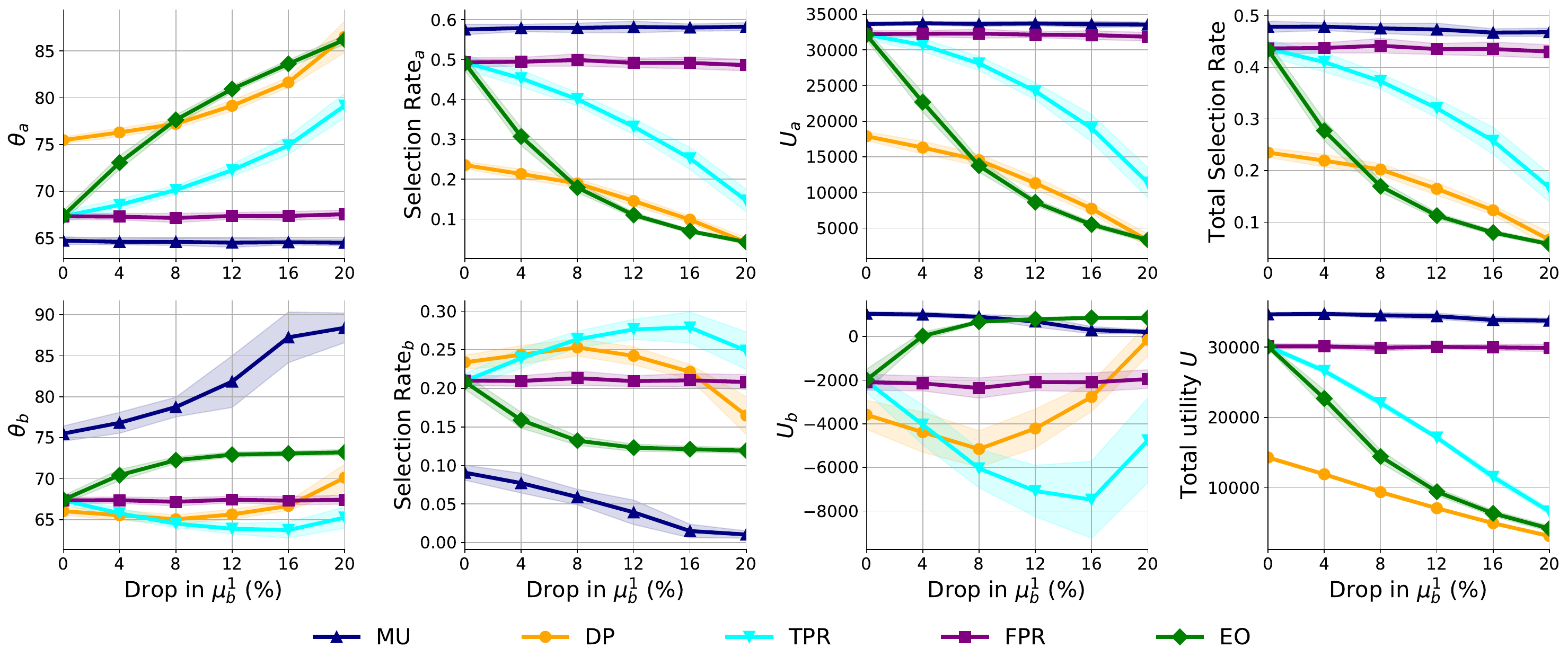}
\caption{Thresholds, selection rates, and utility from synthetic simulation on measurement errors. The results are from 20 runs.}
\label{fig:synthetic-feature_b1-thresh-select-util.}
\end{figure*}

\begin{figure}[h]
\centering
\includegraphics[width=0.45\textwidth]{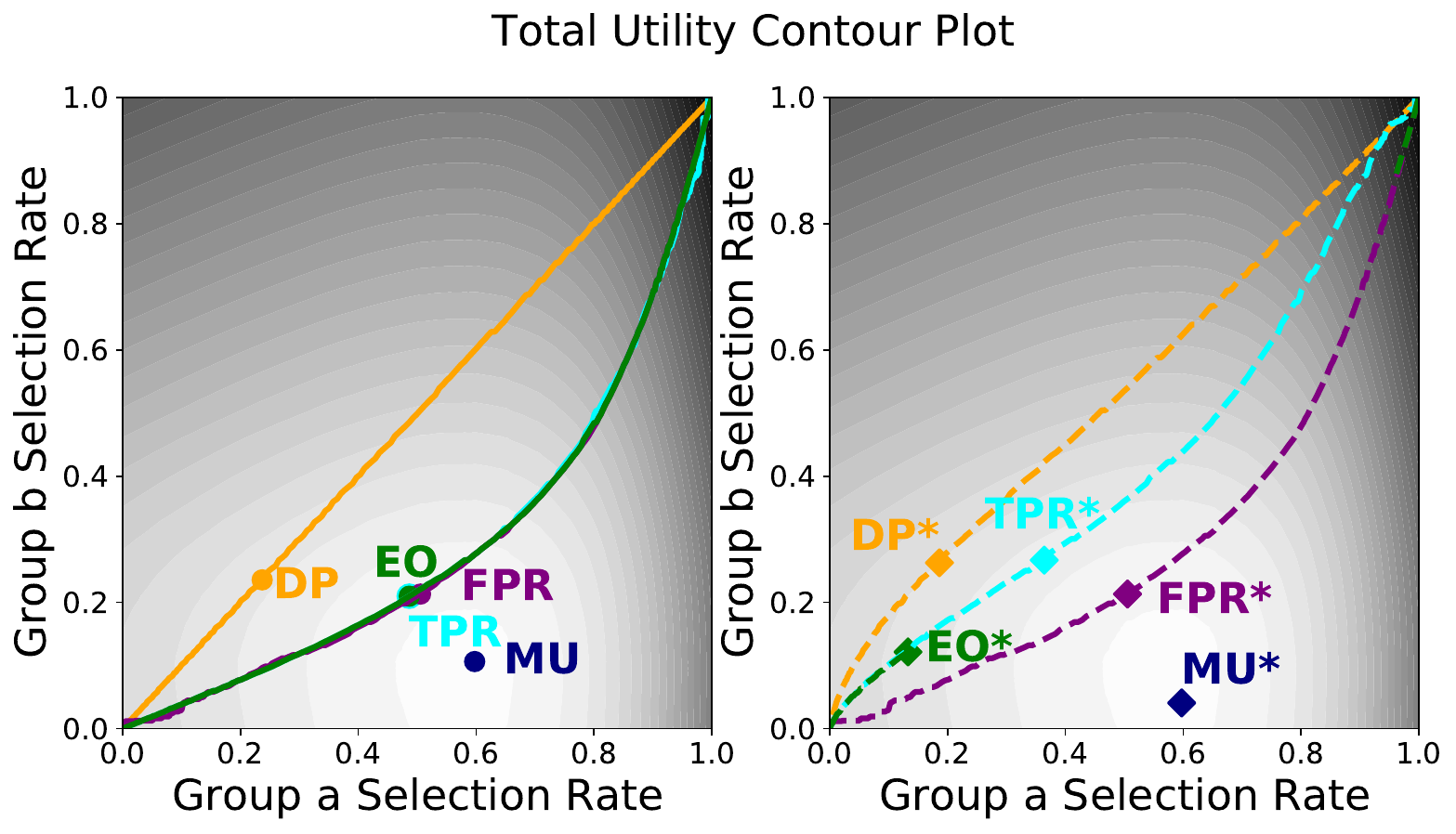}
\caption{Total utility achieved on the synthetic dataset with similar settings to FICO as a function of selection rates. Brighter regions represent higher utility. The curves highlight solutions satisfying the constraints. Left: unbiased data. Right: biased data with 10\% drop in $\mu^1_b$.}
\label{fig:synthetic-feature-util-contour}
\end{figure}

\section{Proofs and Additional Results for Section~\ref{sec:analytical}}\label{app:analytical-proofs}

\subsection{Proof of Lemma~\ref{lemma:MU-opt}}

\begin{proof}
In the absence of a fairness constraint, the firm's utility on each group, $U_g(\theta_g) = \int_{\theta_g}^\infty (\alpha_g u_+f^1_g(x) - (1-\alpha_g)u_-f^0_g(x))\mathrm{d}x$ can be maximized independently. We have
\[\frac{\partial U_g(\theta_g)}{\partial \theta_g} = -(\alpha_g u_+f^1_g(\theta_g) - (1-\alpha_g)u_-f^0_g(\theta_g))~. \]
By Assumption~\ref{ass:MLR} and the above, $\alpha_g u_+f^1_g(\theta^\texttt{MU}_g) = (1-\alpha_g)u_-f^0_g(\theta^\texttt{MU}_g)$ is the maximizer of $U_g(\theta_g)$, establishing the first claim. The expression in terms of the qualification profiles follows from \eqref{eq:gamma-f-alpha}.  
\end{proof}

\subsection{Proof of Lemma~\ref{lemma:MU-bias}}

\begin{proof}
The increase in the decision thresholds on group $b$ follows from the characterizations in Lemma~\ref{lemma:MU-opt}, and noting that the functions $\frac{f_b^1(x)}{f_b^0(x)}$ and $\gamma_b(x)$ are both increasing. The thresholds on group $a$ remain unaffected since the optimization of $U_a(\theta_a)$ and $U_b(\theta_b)$ can be decoupled. The reduction in the firm's utility follows from the proof of Lemma~\ref{lemma:MU-opt}): $\{\theta^{\texttt{MU}}_a, \theta^{\texttt{MU}}_b\}$ is the only maximizer of the firm's utility, and so the utility drops at the perturbed thresholds. 

Finally, the claim that if $\hat{x}=x-\epsilon_b(x)$, where $\epsilon_b(x)\geq 0, \forall x$, then $\frac{\hat{f}^1_b({\theta}^{\texttt{MU}}_b)}{\hat{f}^0_b({\theta}^{\texttt{MU}}_b)} <  \frac{f_b^1({\theta}^{\texttt{MU}}_b)}{f_b^0({\theta}^{\texttt{MU}}_b)}$, follows directly from Assumption~\ref{ass:MLR}. We also note that for the lemma to hold, we in fact only need the underestimation of the feature measurements to happen at the unbiased \texttt{MU} threshold ${\theta}^{\texttt{MU}}_b$. 
\end{proof}

\subsection{Proof of Lemma~\ref{lemma:fair-opt}}

\begin{proof}
Given threshold policies, the firm's problem in \eqref{eq:f-opt-thds} can be further simplified to
\begin{align*}
	\max_{\theta_a, \theta_b} & ~~\sum_{g\in\{a,b\}} n_g \left(\alpha_g u_+(1-F^1_g(\theta_g)) - (1-\alpha_g)u_-(1-F^0_g(\theta_g))\right)\\
	\text{s.t.} & \quad \mathcal{C}_a^{\texttt{f}}(\theta_a) = \mathcal{C}_b^{\texttt{f}}(\theta_b)~.
\end{align*}
We first use the constraint to express $\theta_a$ as a function of $\theta_b$, that is, find a function $\phi^\texttt{f}$ such that $\theta_a = \phi^\texttt{f}(\theta_b)$. To do so, note that for all three constraints $\texttt{f}\in\{\texttt{DP}, \texttt{TPR}, \texttt{FPR}\}$, the function $C^\texttt{f}_g$ is invertible. Therefore, $\theta_a = (C^\texttt{f}_a)^{-1}(C^\texttt{f}_b(\theta_b))$. Further, 
\begin{align}
\frac{\partial \phi^\texttt{f}(\theta_b)}{\partial \theta_b} &= \frac{\partial (C^\texttt{f}_a)^{-1}(C^\texttt{f}_b(\theta_b))}{\partial \theta_b} = \frac{1}{\frac{\partial C_a^\texttt{f}\left( (C^\texttt{f}_a)^{-1}(C^\texttt{f}_b(\theta_b))\right)}{\partial \theta}}\frac{\partial C_b^\texttt{f}(\theta_b)}{\partial \theta}\notag\\
& 
= \frac{\frac{\partial C_b^\texttt{f}(\theta_b)}{\partial \theta}}{\frac{\partial C_a^\texttt{f}(\theta_a)}{\partial \theta}}
\label{eq:a-aafo-b}
\end{align}
It is also easy to show that $\frac{\partial C_g^\texttt{DP}(\theta)}{\partial \theta}=\alpha_gf^1_g(\theta)+(1-\alpha_g)f^0_g(\theta)$, $\frac{\partial C_g^\texttt{TPR}(\theta)}{\partial \theta}=f^1_g(\theta)$, and $\frac{\partial C_g^\texttt{FPR}(\theta)}{\partial \theta}=f^0_g(\theta)$; therefore, $\phi^\texttt{f}(\cdot)$ is an increasing function for all three constraints. 

With the conversion $\theta_a = \phi^\texttt{f}(\theta_b)$, the firm's optimizaiton problem reduces to
\begin{align*}
	\theta_b^* = \arg\max_{\theta_b} & ~~ n_a \Big(\alpha_a u_+(1-F^1_a(\phi^\texttt{f}(\theta_b))) \\
	&- (1-\alpha_a)u_-(1-F^0_a(\phi^\texttt{f}(\theta_b)))\Big) \\
	&+ n_b \Big(\alpha_b u_+(1-F^1_b(\theta_b)) - (1-\alpha_b)u_-(1-F^0_b(\theta_b))\Big)
\end{align*}
The first derivative of the objective function above with respect to $\theta_b$ is given by
\begin{align*}
	\frac{\partial U(\phi^\texttt{f}(\theta_b), \theta_b)}{\partial \theta_b} = &-n_a \frac{\frac{\partial C_b^\texttt{f}(\theta_b)}{\partial \theta}}{\frac{\partial C_a^\texttt{f}(\theta_a)}{\partial \theta}}\Big(\alpha_a u_+f^1_a(\theta_a) - (1-\alpha_a)u_-f^0_a(\theta_a)\Big) \\
	&\qquad - n_b \Big(\alpha_b u_+f^1_b(\theta_b) - (1-\alpha_b)u_-f^0_b(\theta_b)\Big)
\end{align*}
Consider the thresholds $\{\theta^*_a, \theta_b^*\}$ at which this first derivative is zero. Given that $\phi^\texttt{f}$ is increasing, alternative feasible profiles have either both thresholds smaller or larger than those in $\{\theta^*_a, \theta_b^*\}$.  Together with Assumption~\ref{ass:MLR}, this first derivative is positive/negative when the thresholds decrease/increase. We therefore conclude that $\{\theta^*_a, \theta_b^*\}$ maximizes the firm's utility. 

That is, the optimal fairness-constrained thresholds satisfy
\begin{align*}
	& n_a \frac{\frac{\partial C_b^\texttt{f}(\theta^{\texttt{f}}_b)}{\partial \theta}}{\frac{\partial C_a^\texttt{f}(\theta^{\texttt{f}}_a)}{\partial \theta}}\Big(\alpha_a u_+f^1_a(\theta^{\texttt{f}}_a) - (1-\alpha_a)u_-f^0_a(\theta^{\texttt{f}}_a)\Big) \\
	&\quad + n_b \Big(\alpha_b u_+f^1_b(\theta^{\texttt{f}}_b) - (1-\alpha_b)u_-f^0_b(\theta^{\texttt{f}}_b)\Big) = 0
\end{align*}
completing the proof. 
\end{proof}

\subsection{Derivations of thresholds in Table~\ref{t:fair-thds}}\label{app:table-derivations}

In this appendix, we detail the derivation of the thresholds in Tables~\ref{t:fair-thds} and \ref{t:fair-thds-2}. 

\begin{table*}[ht]
    \caption{Optimal fairness-constrained thresholds under different choices of $\texttt{f}\in\{\texttt{DP}, \texttt{TPR}, \texttt{FPR}\}$ fairness criteria in terms of $f_g^y(x)$ and $\alpha_g$.}
    \label{t:fair-thds}
    \centering
    \begin{tabular}{c c }
    \hline
        \texttt{f} &  Optimal thresholds in terms of $f_g^y(x)$ and $\alpha_g$ \\
        \hline
        {\texttt{DP}} & $\sum_{g\in\{a,b\}} n_g \frac{\alpha_g u_+f^1_g(\theta^\texttt{DP}_g) -  (1-\alpha_g)u_-f^0_g(\theta^\texttt{DP}_g)}{\alpha_gf^1_g(\theta^\texttt{DP}_g) +  (1-\alpha_g)f^0_g(\theta^\texttt{DP}_g)} = 0$ \\
         \hline
        {\texttt{TPR}} & $\sum_{g\in\{a,b\}} n_g \left(\alpha_g u_+ -  (1-\alpha_g)u_-\frac{f^0_g(\theta_g^\texttt{TPR})}{f^1_g(\theta_g^\texttt{TPR})}\right) = 0$ \\
         \hline
        {\texttt{FPR}} & $\sum_{g\in\{a,b\}} n_g \left(\alpha_g u_+ \frac{f^1_g(\theta_g^\texttt{FPR})}{f^0_g(\theta_g^\texttt{FPR})}-  (1-\alpha_g)u_-\right) = 0$ \\
         \hline
    \end{tabular}
\end{table*}

\begin{table*}[ht]
    \caption{Optimal fairness-constrained thresholds under different choices of $\texttt{f}\in\{\texttt{DP}, \texttt{TPR}, \texttt{FPR}\}$ fairness criteria in terms of $\gamma_g(x)$ and $\alpha_g$.}
    \label{t:fair-thds-2}
    \centering
    \begin{tabular}{c c}
    \hline
        \texttt{f} &  Optimal thresholds in terms of $\gamma_g(x)$ and $\alpha_g$\\
        \hline
        {\texttt{DP}} & $n_a \gamma_a(\theta_a^{\texttt{DP}}) + n_b \gamma_b(\theta_b^{\texttt{DP}}) = \frac{u_-}{u_++u_-}$\\
         \hline
        {\texttt{TPR}} & $\frac{n_a\alpha_a}{\gamma_a(\theta_a^\texttt{TPR})} + \frac{n_b\alpha_b}{\gamma_b(\theta_b^\texttt{TPR})} = \frac{1}{\frac{u_-}{u_++u_-}}(n_a\alpha_a+n_b\alpha_b)$\\
         \hline
        {\texttt{FPR}} &
        $\frac{n_a(1-\alpha_a)}{1-\gamma_a(\theta_a^\texttt{FPR})} + \frac{n_b(1-\alpha_b)}{1-\gamma_b(\theta_b^\texttt{FPR})} = \frac{1}{1-\frac{u_-}{u_++u_-}}(n_a(1-\alpha_a)+n_b(1-\alpha_b))$\\
        \hline
    \end{tabular}
\end{table*}

These are based on the characterization in Lemma~\ref{lemma:fair-opt}, which showed that the fairness-constrained thresholds satisfy:
\begin{align}
\sum_{g\in\{a,b\}} n_g \frac{\alpha_g u_+f^1_g(\theta^\texttt{f}_g) -  (1-\alpha_g)u_-f^0_g(\theta^\texttt{f}_g)}{\partial C_g^{\texttt{f}}(\theta^\texttt{f}_g)/\partial \theta} = 0~.
    \label{eq:f-opt-thds}
\end{align}

It is straightforward to verify that $\frac{\partial C_g^\texttt{DP}(\theta)}{\partial \theta}=\alpha_gf^1_g(\theta)+(1-\alpha_g)f^0_g(\theta)$, $\frac{\partial C_g^\texttt{TPR}(\theta)}{\partial \theta}=f^1_g(\theta)$, and $\frac{\partial C_g^\texttt{FPR}(\theta)}{\partial \theta}=f^0_g(\theta)$. The first column in Table~\ref{t:fair-thds} then follows directly from the characterizations \eqref{eq:f-opt-thds} of Lemma~\ref{lemma:fair-opt}. We next derive the characterizations in the second column of Table~\ref{t:fair-thds} from the first column together with the relation in \eqref{eq:gamma-f-alpha}. Specifically,\footnote{We will use $\pm$ as a shorthand for a term being added and subtracted.} 

\begin{itemize}
\item For \texttt{DP}:
\begin{align*}
    &\sum_{g\in\{a,b\}} n_g \frac{\alpha_g u_+f^1_g(\theta^\texttt{DP}_g) -  (1-\alpha_g)u_-f^0_g(\theta^\texttt{DP}_g)}{\alpha_gf^1_g(\theta^\texttt{DP}_g) +  (1-\alpha_g)f^0_g(\theta^\texttt{DP}_g)} = 0 \\
    \Leftrightarrow & ~~ \sum_{g\in\{a,b\}} n_g \tfrac{\alpha_g u_+f^1_g(\theta^\texttt{DP}_g) -  (1-\alpha_g)u_-f^0_g(\theta^\texttt{DP}_g) \pm \alpha_g u_-f^1_g(\theta_g^\texttt{DP})}{\alpha_gf^1_g(\theta^\texttt{DP}_g) +  (1-\alpha_g)f^0_g(\theta^\texttt{DP}_g)} = 0 \\
    \Leftrightarrow & ~~ \sum_{g\in\{a,b\}} n_g \left(\frac{\alpha_gf^1_g(\theta^\texttt{DP}_g)}{\alpha_gf^1_g(\theta^\texttt{DP}_g) +  (1-\alpha_g)f^0_g(\theta^\texttt{DP}_g)}(u_++u_-) - u_-\right) = 0 \\
    \Leftrightarrow & ~~ \sum_{g\in\{a,b\}} n_g\gamma_g(\theta^\texttt{DP}_g) = \sum_{g\in\{a,b\}} n_g \frac{u_-}{u_++u_-} \\
    \Leftrightarrow & \quad n_a \gamma_a(\theta_a^{\texttt{DP}}) + n_b \gamma_b(\theta_b^{\texttt{DP}}) = \frac{u_-}{u_++u_-}
\end{align*}
    \item For \texttt{TPR}:
\begin{align*}
    & \sum_{g\in\{a,b\}} n_g \left(\alpha_g u_+ -  (1-\alpha_g)u_-\frac{f^0_g(\theta_g^\texttt{TPR})}{f^1_g(\theta_g^\texttt{TPR})}\right) = 0\\
    \Leftrightarrow & ~~ \sum_{g\in\{a,b\}} n_g \frac{\alpha_g u_+f^1_g(\theta_g^\texttt{TPR}) - (1-\alpha_g)u_- f^0_g(\theta_g^\texttt{TPR})}{f^1_g(\theta_g^\texttt{TPR})} = 0\\
    \Leftrightarrow & ~~ \sum_{g\in\{a,b\}} n_g \tfrac{\alpha_g u_+f^1_g(\theta_g^\texttt{TPR}) - (1-\alpha_g)u_- f^0_g(\theta_g^\texttt{TPR}) \pm \alpha_g u_-f^1_g(\theta_g^\texttt{TPR})}{f^1_g(\theta_g^\texttt{TPR})} = 0\\
    \Leftrightarrow & ~~ \sum_{g\in\{a,b\}} n_g \left(\alpha_g(u_++u_-) - u_-\tfrac{\alpha_g f^1_g(\theta_g^\texttt{TPR}) + (1-\alpha_g)f^0_g(\theta_g^\texttt{TPR})}{f^1_g(\theta_g^\texttt{TPR})}\right) = 0\\
    \Leftrightarrow & ~~ \sum_{g\in\{a,b\}} n_g \left(\alpha_g\frac{u_++u_-}{u_-} - \frac{\alpha_g}{\gamma_g(\theta_g^\texttt{TPR})}\right) = 0\\
    \Leftrightarrow & ~~ \frac{n_a\alpha_a}{\gamma_a(\theta_a^\texttt{TPR})} + \frac{n_b\alpha_b}{\gamma_b(\theta_b^\texttt{TPR})} = \frac{1}{\frac{u_-}{u_++u_-}}(n_a\alpha_a+n_b\alpha_b)
\end{align*}

\item For \texttt{FPR}:
\begin{align*}
    & \sum_{g\in\{a,b\}} n_g \left(\alpha_g u_+ \frac{f^1_g(\theta_g^\texttt{FPR})}{f^0_g(\theta_g^\texttt{FPR})}-  (1-\alpha_g)u_-\right) = 0\\
    \Leftrightarrow & ~~ \sum_{g\in\{a,b\}} n_g \frac{\alpha_g u_+f^1_g(\theta_g^\texttt{FPR})-(1-\alpha_g)u_-f^0_g(\theta_g^\texttt{FPR})}{f^0_g(\theta_g^\texttt{FPR})} = 0\\
    \Leftrightarrow & ~~ \sum_{g\in\{a,b\}} n_g \tfrac{\alpha_g u_+f^1_g(\theta_g^\texttt{FPR})-(1-\alpha_g)u_-f^0_g(\theta_g^\texttt{FPR}) \pm (1-\alpha_g)u_+f^0_g(\theta_g^\texttt{FPR}) }{f^0_g(\theta_g^\texttt{FPR})} = 0\\
    \Leftrightarrow & ~~ \sum_{g\in\{a,b\}} n_g \left(u_+\tfrac{\alpha_g f^1_g(\theta_g^\texttt{FPR}) + (1-\alpha_g)f^0_g(\theta_g^\texttt{FPR}) }{f^0_g(\theta_g^\texttt{FPR})} - (1-\alpha_g)(u_++u_-)\right) = 0\\
    \Leftrightarrow & ~~ \sum_{g\in\{a,b\}} n_g \left(\frac{(1-\alpha_g)}{1-\gamma_g(\theta_g^\texttt{FPR})} - (1-\alpha_g)\frac{u_++u_-}{u_+}\right) = 0\\
    \Leftrightarrow & ~~ \tfrac{n_a(1-\alpha_a)}{1-\gamma_a(\theta_a^\texttt{FPR})} + \tfrac{n_b(1-\alpha_b)}{1-\gamma_b(\theta_b^\texttt{FPR})} = \frac{1}{1-\frac{u_-}{u_++u_-}}(n_a(1-\alpha_a)+n_b(1-\alpha_b))
\end{align*}

\end{itemize}

\subsection{Proof of Proposition~\ref{prop:f-bias-gamma}}\label{app:prop-impact-gamma-proof}
\begin{proof}
We begin with some preliminaries, determining the impacts of $\hat{\gamma}_b(x)=\beta\gamma_b(x)$ on other estimates of the underlying population characteristics. To do so, let $f_g(x):=\mathbb{P}(X=x|G=g) = \alpha_g f_g^1(x)+(1-\alpha_g)f^0_g(x)$ be the feature distribution across all agents (qualified and unqualified) from group $g$. We note that if there is a bias in qualification assessments $\hat{\gamma}_g(x)=\beta \gamma_g(x)$(which affects the labels $y$), these overall feature distribution estimates $\hat{f}_g(x)=\mathbb{P}(X=x|G=g)$ will not be affected; that is $\hat{f}_g(x)=f_g(x)$. However, $\hat{\alpha}_g$ and $\hat{f}^y_g(x)$ can still be impacted. 

We first identify the impacts of the change in $\gamma_g(x)$ on $\alpha_g$. By definition:
    \[\alpha_g = \int_x \mathbb{P}(Y=1|X=x,G=g)\mathbb{P}(X=x|G=g)\mathrm{d}x~.\]
Therefore, if $\hat{\gamma}_b(x)=\beta\gamma_b(x), \forall x$, we have $\hat{\alpha}_b = \beta\alpha_b$. 

Now, by definition and the Bayes' rule
\[f_g^y(x) = \mathbb{P}(X=x|Y=y, G=g) = \tfrac{\mathbb{P}(Y=y|X=x, G=g)\mathbb{P}(X=x| G=g)}{\mathbb{P}(Y=y| G=g)}~.\]
Therefore, 
\begin{align*}
    f_g^1(x) &= \frac{\gamma_g(x)f_g(x)}{\alpha_g}~.\\
    f_g^0(x) &= \frac{(1-\gamma_g(x))f_g(x)}{(1-\alpha_g)}~.
\end{align*}

Noting that $\hat{\alpha}=\beta_g\alpha_g$, combined with the above relations, we conclude that $\hat{f}^1_g(x)=f^1_g(x)$, while $\hat{f}^0_g(x)=\frac{1-\alpha_g}{1-\beta\alpha_g} \frac{1-\beta\gamma_g(x)}{1-\gamma_g(x)} f^0_g(x)$. Intuitively, this is expected: $\hat{\gamma}_g(x)=\beta \gamma_g(x)$ can be viewed as flipping label 1 to label 0 in the training data with probability $\beta$. This leaves the feature distribution of qualified agents unchanged (as the flipping probability is independent of the feature $x$), whereas it adds (incorrect) data to the feature distribution of unqualified agents, hence biasing $f^0_g(x)$. 

We now proceed with the proof of the proposition. 

\emph{Part (i):} 

    $\bullet$ For \texttt{DP}: The firm picks the thresholds such that $h^\texttt{DP}(\hat{\theta}_a^{\texttt{DP}},\hat{\theta}_b^{\texttt{DP}}, \beta)=0$, where \[h^\texttt{DP}({\theta}_a,{\theta}_b, \beta) : = n_a\gamma_a({\theta}_a) + n_b \beta\gamma_b({\theta}_b) - \frac{u_-}{u_++u_-}~.\]
    Note that as $\gamma_g(x)$ are increasing functions under Assumption~\ref{ass:MLR}, $h^\texttt{DP}$ is increasing in both thresholds. It is also increasing in $\beta$.  
    
    For $\beta\in(0,1)$,  $h^\texttt{DP}({\theta}_a^{\texttt{DP}},{\theta}_b^{\texttt{DP}},\beta)< 0$. Therefore, to attain  $h^\texttt{DP}(\hat{\theta}_a^{\texttt{DP}},\hat{\theta}_b^{\texttt{DP}},\beta)=0$, at least one of the thresholds should increase compared to the unbiased thresholds $\{{\theta}_a^{\texttt{DP}},{\theta}_b^{\texttt{DP}}\}$. In addition, the \texttt{DP}-constrained thresholds when data is biased are selected such that $\int_{\hat{\theta}^{\texttt{DP}}_a}^\infty {f}_a(x)\mathrm{d}x=\int_{\hat{\theta}^{\texttt{DP}}_b}^\infty \hat{f}_b(x)\mathrm{d}x$ (i.e., based on the biased training data); since $\hat{f}_b(x)=f_b(x)$, we conclude that the changes in the thresholds are aligned (i.e, either both decrease or both increase compared to the unbiased case). We conclude that $\hat{\theta}^{\texttt{DP}}_g(\beta)\geq {\theta}^{\texttt{DP}}_g$ for both groups.  
    
    $\bullet$ For \texttt{TPR}: from Table~\ref{t:fair-thds}, The firm picks the thresholds such that $h^\texttt{TPR}(\hat{\theta}_a^{\texttt{TPR}},\hat{\theta}_b^{\texttt{TPR}}, \beta)=0$, where \[h^\texttt{TPR}({\theta}_a,{\theta}_b, \beta) : = \frac{n_a\alpha_a}{\gamma_a(\theta_a^\texttt{TPR})} + \frac{n_b\alpha_b}{\gamma_b(\theta_b^\texttt{TPR})} - \tfrac{1}{\frac{u_-}{u_++u_-}}(n_a\alpha_a+n_b\beta\alpha_b)~.\]
    Note that as $\gamma_g(x)$ are increasing functions under Assumption~\ref{ass:MLR}, $h^\texttt{TPR}$ is decreasing in both thresholds. It is also decreasing in $\beta$. 
    
    For $\beta\in(0,1)$,  $h^\texttt{TPR}({\theta}_a^{\texttt{TPR}},{\theta}_b^{\texttt{TPR}}, \beta)> 0$. Therefore, to attain  $h^\texttt{TPR}(\hat{\theta}_a^{\texttt{TPR}},\hat{\theta}_b^{\texttt{TPR}}, \beta)=0$, at least one of the thresholds should increase compared to the unbiased thresholds $\{{\theta}_a^{\texttt{TPR}},{\theta}_b^{\texttt{TPR}}\}$. In addition, as $\hat{f}^1_g(x)=f^1_g(x)$, and by the definition of \texttt{TPR}, the changes in the thresholds are aligned (i.e, either both decrease or both increase compared to the unbiased case). Therefore, $\hat{\theta}^{\texttt{TPR}}_g(\beta)\geq {\theta}^{\texttt{TPR}}_g$ for both groups.
    
    $\bullet$ For \texttt{FPR}: based on Table~\ref{t:fair-thds}, The firm picks the thresholds such that $h^\texttt{FPR}(\hat{\theta}_a^{\texttt{FPR}},\hat{\theta}_b^{\texttt{FPR}}, \beta)=0$, where \begin{align*}
        h^\texttt{FPR}({\theta}_a,{\theta}_b, \beta) : = & \tfrac{n_a(1-\alpha_a)}{1-\gamma_a(\theta_a)} + \tfrac{n_b(1-\beta\alpha_b)}{1-\beta\gamma_b(\theta_b)} \notag\\
        &- \frac{1}{1-\frac{u_-}{u_++u_-}}(n_a(1-\alpha_a)+n_b(1-\beta\alpha_b))
    \end{align*}
    Note that as $\gamma_g(x)$ are increasing functions under Assumption~\ref{ass:MLR}, $h^\texttt{FPR}$ is increasing in both thresholds. To identify its trend in $\beta$, 
    the derivative of $h^\texttt{FPR}$ with respect to $\beta$ is given by
    \begin{align*}
        \frac{\partial h^\texttt{FPR}}{\partial \beta} = &  n_b\frac{-\alpha_b(1-\beta\gamma_b(\theta_b))+\gamma_b(\theta_b)(1-\beta\alpha_b)}{(1-\beta\gamma_b(\theta_b))^2} + \frac{n_b\alpha_b}{1-\frac{u_-}{u_++u_-}}\\
        & = \alpha_b n_b \left(\frac{1}{1-\frac{u_-}{u_++u_-}}-\frac{1}{1-\beta\gamma_b(\theta_b)}\right) + \frac{n_b\gamma_b(\theta_b)(1-\beta\alpha_b)}{(1-\beta\gamma_b(\theta_b))^2} 
    \end{align*}
    As $h^\texttt{FPR}(\hat{\theta}^\texttt{FPR}_a, \hat{\theta}^\texttt{FPR}_b, \beta) = 0$, and $\gamma_b(x)\leq \gamma_a(x)$, we can conclude that $\frac{1}{1-\beta\gamma_b(\hat{\theta}^\texttt{FPR}_b)}\leq \frac{1}{1-\frac{u_-}{u_++u_-}} \leq \frac{1}{1-\gamma_a(\hat{\theta}^\texttt{FPR}_a)}$. This means that $\frac{\partial h^\texttt{FPR}}{\partial \beta}$ is increasing at the optimal biased thresholds for each $\beta\in(0,1)$. Therefore, $h^\texttt{FPR}({\theta}_a^{\texttt{FPR}},{\theta}_b^{\texttt{FPR}}, \beta)\leq 0$ as $\beta$ decreases from 1. That means that in order to attain $h^\texttt{FPR}(\hat{\theta}_a^{\texttt{FPR}},\hat{\theta}_b^{\texttt{FPR}}, \beta)=0$, at least one of the thresholds should increase compared to the unbiased thresholds $\{{\theta}_a^{\texttt{FPR}},{\theta}_b^{\texttt{FPR}}\}$. 
    In addition, from \eqref{eq:a-aafo-b} in the proof of Lemma~\ref{lemma:fair-opt}, we know that given that $\beta=1$, if $\theta^{\texttt{FPR}}_a$ drops, so should $\theta^{\texttt{FPR}}_b$ for the \texttt{FPR} constraint to continue to hold.
    So it must be that both thresholds increase. We conclude that both thresholds should increase, that is $\hat{\theta}^{\texttt{FPR}}_g(\beta)\geq {\theta}^{\texttt{FPR}}_g$ for both groups when $\beta$ drops from 1. 

Lastly, for all three constraints, applying the same argument to levels of qualification assessment biases $\beta^1>\beta^2$, we conclude that  $\hat{\theta}^{\texttt{f}}_g(\beta^2)>\hat{\theta}^{\texttt{f}}_g(\beta^1)$. That is, $\hat{\theta}^{\texttt{f}}_g(\beta)$ is decreasing in $\beta$.

\paragraph{Part (ii):} The \texttt{DP}-constrained thresholds when data is biased are selected such that $\int_{\hat{\theta}^{\texttt{DP}}_a}^\infty {f}_a(x)\mathrm{d}x=\int_{\hat{\theta}^{\texttt{DP}}_b}^\infty \hat{f}_b(x)\mathrm{d}x$ (i.e., based on the biased training data); since $\hat{f}_b(x)=f_b(x)$, this constraint is also satisfied on the unbiased data. That is, \texttt{DP} continues to hold (on the unbiased training data, as intended) at  $\{\hat{\theta}^{\texttt{DP}}_a,\hat{\theta}^{\texttt{DP}}_b\}$. 

Next, as $\hat{f}^1_g(x)= f^1_g(x)$ and  $\hat{f}^0_b(x)=\frac{1-\alpha_b}{1-\beta\alpha_b} \frac{1-\beta\gamma_b(x)}{1-\gamma_b(x)} f^0_b(x)$, the thresholds satisfying \texttt{TPR} on the biased data also satisfy \texttt{TPR} on the unbiased data, while the same is not true for \texttt{FPR}.

\paragraph{Part (iii):} We note that by the previous part, the new thresholds $\{\hat{\theta}^{\texttt{f}}_a(\beta), \hat{\theta}^{\texttt{f}}_b(\beta)\}$ continue to satisfy the \texttt{DP} and \texttt{TPR} constraints. That means that these constraints were also feasible choices in the unbiased scenario. Therefore, $U(\hat{\theta}^{\texttt{f}}_a(\beta), \hat{\theta}^{\texttt{f}}_b(\beta))\leq U({\theta}^{\texttt{f}}_a, {\theta}^{\texttt{f}}_b)$ as $\{{\theta}^{\texttt{f}}_a, {\theta}^{\texttt{f}}_b\}$ are the expected utility maximizers among all feasible threshold choices in the unconstrained case.

\paragraph{Part (iv):} We prove this part by construction. An example of utility drop under FPR is detailed below. An example of utility increase under FPR is given Figure~\ref{fig:fico-thresh-select-util} in our numerical experiments; see Appendix~\ref{app:extra-fico}.

\subsubsection{Case where the firm's utility decreases under \texttt{FPR}}\label{app:synthetic-label-bias-FPR-decrease}
The group ratios and qualification rates are set as follows: $n_a$ = 0.5, $n_b$ = 0.5, $\alpha_a$ = 0.5, and $\alpha_b$ = 0.4. The feature $x$ of each agent is sampled from Gaussian with $\mu$ = 70 and $\sigma$ = 10 if qualified, and from Gaussian with $\mu$ = 50 and $\sigma$ = 10 if unqualified. A total of 100000 examples are sampled. $\frac{u_-}{u_+}$ is set to 10.

\begin{figure}[h]
\centering
\begin{subfigure}[t]{0.45\textwidth}
\centering
\includegraphics[width=\linewidth]{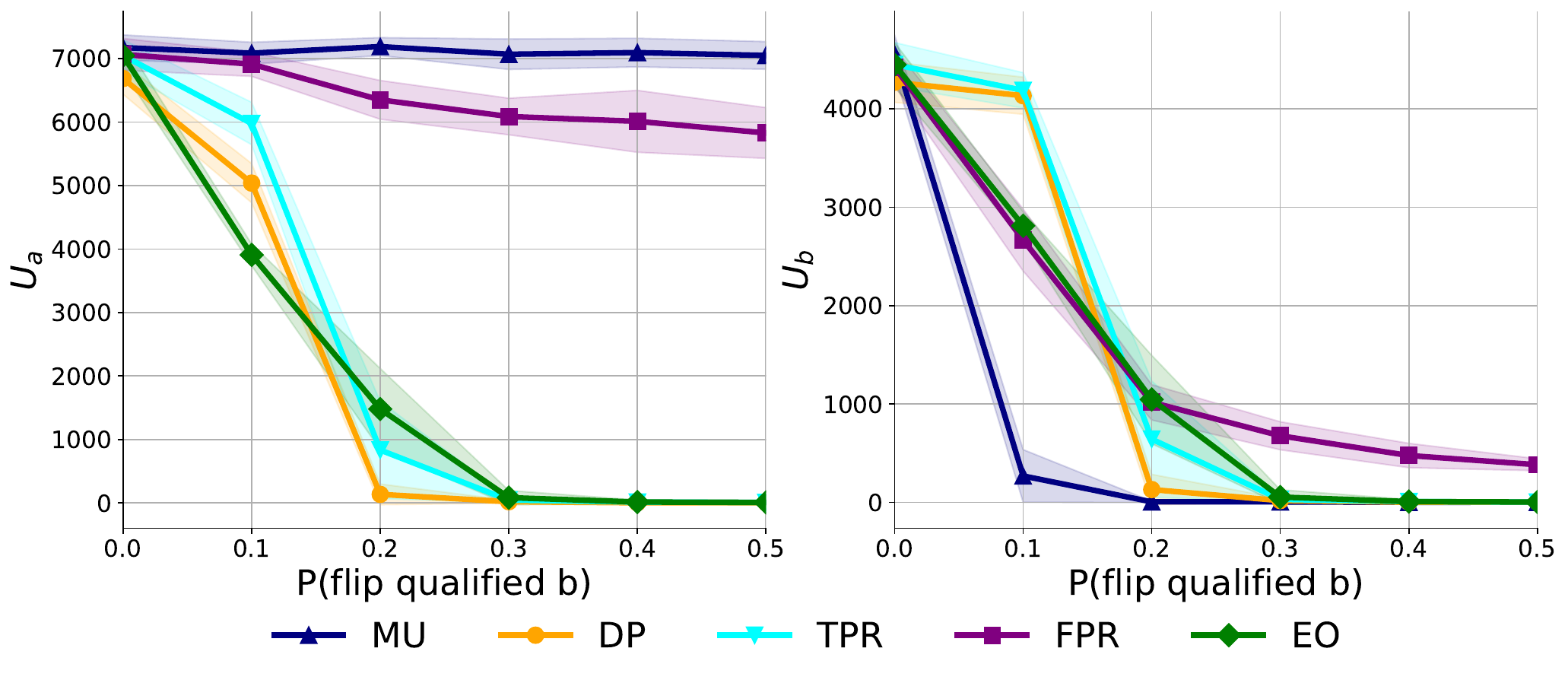}
\caption{Group utility.}
\label{fig:synthetic-FPR-util-drop-groupwise}
\end{subfigure}
\begin{subfigure}[t]{0.45\textwidth}
\centering
\includegraphics[width=\linewidth]{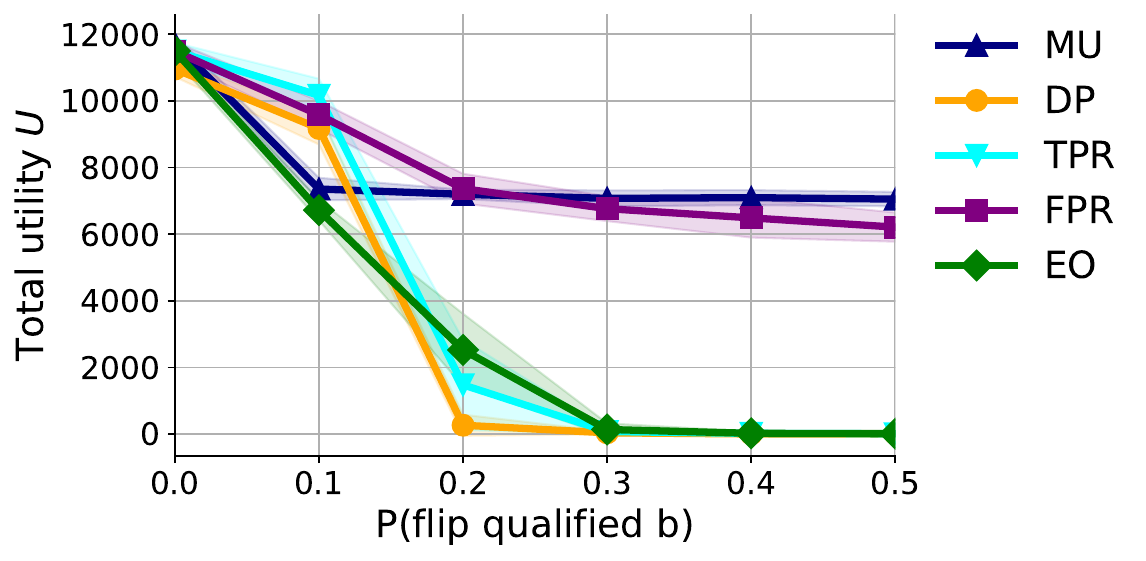}
\caption{Total utility.}
\label{fig:synthetic-FPR-util-drop-total}
\end{subfigure}
\caption{Group utility and total utility in a case where the firm's utility decreases under \texttt{FPR}. Results are from 20 runs.}
\label{fig:synthetic-FPR-util-drop-all}
\end{figure}
\end{proof}

\subsection{Why can the firm's utility increase under \texttt{FPR} relative to \texttt{MU} when data is biased?}
To provide additional intuition for the increase of the firm's utility under the \texttt{FPR} constraint when data is biased, note that the firm's expected utility on group $g$ can be written as
\[U_g(\theta) = \int_{\theta}^\infty [\gamma_g(x)(u_++u_-)-u_-]f_g(x)\mathrm{d}x~.\]
Therefore, if $\gamma_g(x)\geq \frac{u_-}{u_++u_-}$, the utility is decreasing in the threshold, and vice versa. 

From the proof of part (i) of Proposition~\ref{prop:f-bias-gamma}, we know that $\frac{1}{1-\beta\gamma_b(\hat{\theta}^\texttt{FPR}_b)}\leq \frac{1}{1-\frac{u_-}{u_++u_-}} \leq \frac{1}{1-\gamma_a(\hat{\theta}^\texttt{FPR}_a)}$. Therefore, $\gamma_a(\hat{\theta}^\texttt{FPR}_a)\geq \frac{u_-}{u_++u_-}$, and as the threshold on group $a$ increases given the increase in $\hat{\theta}^\texttt{FPR}_a$ with increasing noise, the utility that the firm derives from group $a$ drops. On the other hand, given  $\frac{1}{1-\beta\gamma_b(\hat{\theta}^\texttt{FPR}_b)}\leq \frac{1}{1-\frac{u_-}{u_++u_-}}$, we know that $\beta\gamma_b(\hat{\theta}^\texttt{FPR}_b)\leq \frac{u_-}{u_++u_-}$. However, for sufficiently small $\beta$, it might be that the above holds, while $\gamma_b(\hat{\theta}^\texttt{FPR}_b)\geq \frac{u_-}{u_++u_-}$. Therefore, the utility from group $b$ may increase at large $\beta$ (in a way that the firm's overall utility increases), and may then drop at sufficiently small $\beta$ (so that the firm's overall utility decreases).

\subsection{Sensitivity of \texttt{DP/TPR} thresholds to qualification assessment biases}\label{app:f-bias-gamma-sensitivity}

\begin{proposition}
\label{prop:f-bias-gamma-sensitivity}
Consider the same setting as Proposition~\ref{prop:f-bias-gamma}. Then, the rate of change of group $b$'s thresholds at $\beta=1$ is given by
\begin{align*}
    \tfrac{\partial \hat{\theta}^{\texttt{DP}}_b(\beta)}{\partial \beta}|{_{\beta=1}} &= - \tfrac{1}{\frac{n_a}{n_b} \frac{{f}_b({\theta}^\texttt{DP}_b)}{f_a({\theta}^\texttt{DP}_a)}\frac{\gamma'_a({\theta}_a^{\texttt{DP}})}{\gamma_b({\theta}_b^{\texttt{DP}})} +  \frac{\gamma'_b({\theta}_b^{\texttt{DP}})}{\gamma_b({\theta}_b^{\texttt{DP}})}}, \notag\\
    \tfrac{\partial \hat{\theta}^{\texttt{TPR}}_b(\beta)}{\partial \beta}|{_{\beta=1}} &= -\tfrac{1 + \frac{u_+}{u_-}}{\frac{n_a}{n_b}\frac{\alpha_a}{\alpha_b}\frac{f^1_b({\theta}_b^{\texttt{TPR}})}{f^1_a({\theta}_a^{\texttt{TPR}})}\frac{\gamma'_a({\theta}_a^{\texttt{TPR}})}{(\gamma_a({\theta}_a^{\texttt{TPR}}))^2} + \frac{\gamma'_b({\theta}_b^{\texttt{TPR}})}{(\gamma_b({\theta}_b^{\texttt{TPR}}))^2}}~.
\end{align*}
Further, the rate of change of group $a$'s thresholds at $\beta=1$ is given by
\begin{align*}
    \tfrac{\partial \hat{\theta}^{\texttt{DP}}_a(\beta)}{\partial \beta}|{_{\beta=1}} &= \tfrac{f_b({\theta}_b^{\texttt{DP}})}{f_a({\theta}_a^{\texttt{DP}})} \tfrac{\partial \hat{\theta}^{\texttt{DP}}_b(\beta)}{\partial \beta}|{_{\beta=1}}, \notag\\   \tfrac{\partial \hat{\theta}^{\texttt{TPR}}_a(\beta)}{\partial \beta}|{_{\beta=1}} &= \tfrac{f^1_b({\theta}_b^{\texttt{TPR}})}{f^1_a({\theta}_a^{\texttt{TPR}})} \tfrac{\partial \hat{\theta}^{\texttt{TPR}}_b(\beta)}{\partial \beta}|{_{\beta=1}}~.
\end{align*}
\end{proposition}

As consistent with Proposition~\ref{prop:f-bias-gamma}, this proposition shows that the thresholds increase relative to the unbiased case when qualification assessments become biased (i.e., $\beta$ decreases from 1).\footnote{We note that our proof characterizes the sensitivities at \emph{all} bias rates $\beta$; we highlight the sensitivity at $\beta\rightarrow 1$ to provide intuition for low levels of bias.} The proposition also allows us to assess these sensitivities based on other problem parameters:

$\bullet$ The \texttt{DP} fairness constraint only evaluates the selection rates from each group (regardless of qualifications). Therefore, we see that the increase in group $b$'s threshold is sharper if $\frac{n_a}{n_b}$ is smaller (i.e., the representation of group $b$ increases). This is consistent with the following intuition: the impacts of biases are more pronounced when the effect of this group on the firm's utility increases. 

$\bullet$ For the \texttt{TPR} fairness constraint on the other hand, the increase in group $b$'s threshold is impacted by additional problem parameters: it is higher if $\frac{n_a}{n_b}$ or $\frac{\alpha_a}{\alpha_b}$ is smaller (i.e., the representation or qualification rate of group $b$ is high), or if $\frac{u_+}{u_-}$ is larger (i.e., the firm's benefit/loss increases/decreases). The additional impact of qualification rates and firm's benefit/loss is due to the fact that unlike \texttt{DP}, \texttt{TPR} accounts for the qualification of those selected. Therefore, if qualification rates of group $b$ or benefits from accepting qualified individuals are high, the firm will have to make sharper adjustments to the threshold on group $b$ if it perceives they are less qualified. 

\begin{proof}
We start with \texttt{DP}. First, note that the thresholds satisfy the DP constraint
\[\int_{\hat{\theta}^{\texttt{DP}}_a}^\infty {f}_a(x)\mathrm{d}x=\int_{\hat{\theta}^{\texttt{DP}}_b}^\infty \hat{f}_b(x)\mathrm{d}x = \int_{\hat{\theta}^{\texttt{DP}}_b}^\infty {f}_b(x)\mathrm{d}x~.\]
We can use this constraint to express $\hat{\theta}^\texttt{DP}_a$ as a function of $\hat{\theta}^\texttt{DP}_b$. In particular, following steps similar to those of the proof of Lemma~\ref{lemma:fair-opt}, we will get that
\[\frac{\partial \hat{\theta}^\texttt{DP}_a}{\partial \hat{\theta}^\texttt{DP}_b} = \frac{{f}_b(\hat{\theta}^\texttt{DP}_b)}{f_a(\hat{\theta}^\texttt{DP}_a)}\]

We next note that the firm picks the thresholds such that  \[h^\texttt{DP}(\hat{\theta}_a^{\texttt{DP}},\hat{\theta}_b^{\texttt{DP}}, \beta) = n_a\gamma_a(\hat{\theta}_a^{\texttt{DP}}) + n_b \beta\gamma_b(\hat{\theta}_b^{\texttt{DP}}) - \frac{u_-}{u_++u_-} = 0~.\]
We now invoke the implicit function theorem. First, note that $h^\texttt{DP}({\theta}_a,{\theta}_b, \beta) : = n_a\gamma_a({\theta}_a) + n_b \beta\gamma_b({\theta}_b) - \frac{u_-}{u_++u_-}$ is a continuously differentiable function, with $h^\texttt{DP}(\hat{\theta}_a^{\texttt{DP}},\hat{\theta}_b^{\texttt{DP}}, \beta)=0$ . Further, 
\[\frac{\partial h^\texttt{DP}(\hat{\theta}_a^{\texttt{DP}},\hat{\theta}_b^{\texttt{DP}}, \beta)}{\partial \hat{\theta}_b} = n_a\gamma'_a(\hat{\theta}_a^{\texttt{DP}}) \frac{{f}_b(\hat{\theta}^\texttt{DP}_b)}{f_a(\hat{\theta}^\texttt{DP}_a)} + n_b \beta\gamma'_b(\hat{\theta}_b^{\texttt{DP}})~,\]
where $\gamma'_g(\cdot)$ is the first derivative of $\gamma_g(\cdot)$ and is non-zero as $\gamma_g(\cdot)$ is a strictly increasing function by Assumption~\ref{ass:MLR}. Therefore, this derivative is non-zero (hence, invertible). 

Therefore, by the implicit function theorem, there is an open interval $U$ containing $\beta$, such that within this interval 
\[\frac{\partial \hat{\theta}^\texttt{DP}_b(\beta)}{\partial \beta}
= - \frac{\frac{\partial h^\texttt{DP}(\hat{\theta}_a^{\texttt{DP}},\hat{\theta}_b^{\texttt{DP}}, \beta)}{\partial \beta}}{\frac{\partial h^\texttt{DP}(\hat{\theta}_a^{\texttt{DP}},\hat{\theta}_b^{\texttt{DP}}, \beta)}{\partial \hat{\theta}_b}}
= -\frac{n_b \gamma_b(\hat{\theta}_b^{\texttt{DP}})}{n_a\gamma'_a(\hat{\theta}_a^{\texttt{DP}}) \frac{{f}_b(\hat{\theta}^\texttt{DP}_b)}{f_a(\hat{\theta}^\texttt{DP}_a)} + n_b \beta\gamma'_b(\hat{\theta}_b^{\texttt{DP}})}~.\]
The above holds for all $\beta$. Stating it for $\beta=1$, we get
\begin{align*}
    \frac{\partial \hat{\theta}^{\texttt{DP}}_b(\beta)}{\partial \beta}|{_{\beta=1}} &= - \frac{\gamma_b({\theta}_b^{\texttt{DP}})}{\frac{n_a}{n_b}\gamma'_a({\theta}_a^{\texttt{DP}}) \frac{{f}_b({\theta}^\texttt{DP}_b)}{f_a({\theta}^\texttt{DP}_a)} +  \gamma'_b({\theta}_b^{\texttt{DP}})}
\end{align*}

The proof for the case of \texttt{TPR} is similar. In particular, following steps similar to those of the proof of Lemma~\ref{lemma:fair-opt}, we will get that
\[\frac{\partial \hat{\theta}^\texttt{TPR}_a}{\partial \hat{\theta}^\texttt{TPR}_b} = \frac{{f}^1_b(\hat{\theta}^\texttt{TPR}_b)}{f^1_a(\hat{\theta}^\texttt{TPR}_a)}~.\]
Further, the function determining the \texttt{TPR}-constrained thresholds is given by
\[h^\texttt{TPR}({\theta}_a,{\theta}_b, \beta) : = \frac{n_a\alpha_a}{\gamma_a(\theta_a^\texttt{TPR})} + \frac{n_b\alpha_b}{\gamma_b(\theta_b^\texttt{TPR})} - \tfrac{1}{\frac{u_-}{u_++u_-}}(n_a\alpha_a+n_b\beta\alpha_b)~.\]
Invoking the implicit function theorem for this function, we get
\begin{align*}
\frac{\partial \hat{\theta}^\texttt{TPR}_b(\beta)}{\partial \beta}
&= - \frac{\frac{\partial h^\texttt{TPR}(\hat{\theta}_a^{\texttt{TPR}},\hat{\theta}_b^{\texttt{TPR}}, \beta)}{\partial \beta}}{\frac{\partial h^\texttt{TPR}(\hat{\theta}_a^{\texttt{TPR}},\hat{\theta}_b^{\texttt{TPR}}, \beta)}{\partial \hat{\theta}_b}} \\
&= - \frac{(1+\frac{u_+}{u_-})n_b\alpha_b}{n_a\alpha_a\frac{\gamma'_a(\hat{\theta}_a^{\texttt{TPR}})}{(\gamma_a(\hat{\theta}_a^{\texttt{TPR}}))^2}\frac{f^1_b(\hat{\theta}_b^{\texttt{TPR}})}{f^1_a(\hat{\theta}_a^{\texttt{TPR}})} + n_b\alpha_b \frac{\gamma'_b(\hat{\theta}_b^{\texttt{TPR}})}{(\gamma_b(\hat{\theta}_b^{\texttt{TPR}}))^2}}~.
\end{align*}
For $\beta=1$, this reduces to
\begin{align*}
    \frac{\partial \hat{\theta}^{\texttt{TPR}}_b(\beta)}{\partial \beta}|{_{\beta=1}} &= -\frac{1 + \frac{u_+}{u_-}}{\frac{n_a}{n_b}\frac{\alpha_a}{\alpha_b}\frac{f^1_b({\theta}_b^{\texttt{TPR}})}{f^1_a({\theta}_a^{\texttt{TPR}})}\frac{\gamma'_a({\theta}_a^{\texttt{TPR}})}{(\gamma_a({\theta}_a^{\texttt{TPR}}))^2} + \frac{\gamma'_b({\theta}_b^{\texttt{TPR}})}{(\gamma_b({\theta}_b^{\texttt{TPR}}))^2}}
\end{align*}

\end{proof}

\subsection{The impacts of labeling biases in the advantaged group}\label{app:gamma-biases-advantaged}

\begin{proposition}
[Impact of qualification assessment biases in the advantaged group on \texttt{DP/TPR/FPR} thresholds]\label{prop:f-bias-gamma-a}
Assume the qualification profile of group $a$ is overestimated so that $\hat{\gamma}_a(x)=(1-\beta){\gamma_a(x)} + \beta$, with $\beta\in[0,1)$. Let ${\theta}^{\texttt{f}}_g$ and $\hat{\theta}^{\texttt{f}}_g(\beta)$ denote the optimal decision thresholds satisfying fairness constraint $\texttt{f}\in\{\texttt{DP}, \texttt{TPR}, \texttt{FPR}\}$, obtained from unbiased data and from data with biases on group $a$ given $\beta$, respectively. Then,

~~ (i) $\hat{\theta}^{\texttt{f}}_g(\beta)\geq {\theta}^{\texttt{f}}_g$  for $g\in\{a,b\}, ~\texttt{f}\in\{\texttt{DP}, \texttt{TPR}, \texttt{FPR}\}, ~ \beta\in(0,1]$. Further, $\hat{\theta}^{\texttt{f}}_g(\beta)$ is decreasing in $\beta$. 
    
~~ (ii) The \texttt{DP} and \texttt{FPR}  criteria continue to be met, while \texttt{TPR} is violated, at their $\{\hat{\theta}^{\texttt{f}}_a(\beta), \hat{\theta}^{\texttt{f}}_b(\beta)\}$. 
\end{proposition}

\begin{proof}
We begin with some preliminaries, determining the impacts of $\hat{\gamma}_a(x)=(1-\beta)\gamma_a(x) + \beta$ on other estimates of the underlying population characteristics. To do so, let $f_g(x):=\mathbb{P}(X=x|G=g) = \alpha_g f_g^1(x)+(1-\alpha_g)f^0_g(x)$ be the feature distribution across all agents (qualified and unqualified) from group $g$. Similar to the proof of Proposition~\ref{prop:f-bias-gamma}, we can conclude that $\hat{f}_g(x)=f_g(x)$, yet $\hat{\alpha}_g$ and $\hat{f}^y_g(x)$ can still be impacted, and that by definition of $\alpha_g$, $\hat{\gamma}_a(x)=(1-\beta)\gamma_a(x) + \beta, \forall x$, we have $\hat{\alpha}_a = (1-\beta)\alpha_a + \beta$. 

Now, by definition and the Bayes' rule, we again obtain
\begin{align*}
    f_g^1(x) &= \frac{\gamma_g(x)f_g(x)}{\alpha_g}~.\\
    f_g^0(x) &= \frac{(1-\gamma_g(x))f_g(x)}{(1-\alpha_g)}~.
\end{align*}
Noting that now $\hat{\alpha}_g=(1-\beta)\alpha_g + \beta$, combined with the above relations, we conclude that $\hat{f}^0_g(x)=f^0_g(x)$, while $\hat{f}^1_g(x)=\frac{\alpha_g}{(1-\beta)\alpha_g+\beta} \frac{(1-\beta)\gamma_g(x)+\beta}{\gamma_g(x)} f^1_g(x)$. 

Intuitively, this is again expected: $\hat{\gamma}_a(x)=(1-\beta) \gamma_a(x) + \beta$ can be viewed as flipping label 0 to label 1 in the training data with probability $\beta$. To see why, note that we have
\begin{align*}
    \hat{\gamma}_g(x) &= \mathbb{P}(\hat{Y}=1|G=g, X=x)\\
    & = \mathbb{P}(\hat{Y}=1|G=g, X=x, Y=1)\mathbb{P}({Y}=1|G=g, X=x) \\
    & ~~ + \mathbb{P}(\hat{Y}=1|G=g, X=x, Y=0)\mathbb{P}({Y}=0|G=g, X=x) \\
    & = 1\cdot \gamma_g(x) + \beta \cdot (1-\gamma_g(x)) = (1-\beta) \gamma_g(x) + \beta~.
\end{align*}
This type of change in the labels from 0 to 1 leaves the feature distribution of unqualified agents unchanged (as the flipping probability is independent of the feature $x$), whereas it adds (incorrect) data to the feature distribution of qualified agents, hence biasing $f^1_a(x)$. 

We now proceed with the proof of the proposition. 

\emph{Part (i):} 

    $\bullet$ For \texttt{DP}, the argument is identical to that of Proposition~\ref{prop:f-bias-gamma}, since $\hat{f}_g(x)=f_g(x)$ for both groups.  
    
    $\bullet$ For \texttt{FPR}: based on Table~\ref{t:fair-thds}, The firm picks the thresholds such that $h^\texttt{FPR}(\hat{\theta}_a^{\texttt{FPR}},\hat{\theta}_b^{\texttt{FPR}}, \beta)=0$, where \begin{align*}
        h^\texttt{FPR}({\theta}_a,{\theta}_b, \beta) : = & \tfrac{n_a(1-\alpha_a)}{1-\gamma_a(\theta_a)} + \tfrac{n_b(1-\alpha_b)}{1-\gamma_b(\theta_b)} - \frac{1}{1-\frac{u_-}{u_++u_-}}(n_a((1-\beta)\alpha_a+\beta)+n_b(1-\alpha_b))
    \end{align*}
    Note that as $\gamma_g(x)$ are increasing functions under Assumption~\ref{ass:MLR}, $h^\texttt{FPR}$ is increasing in both thresholds. It is also decreasing in $\beta$. 
    
    For $\beta\in(0,1)$,  $h^\texttt{FPR}({\theta}_a^{\texttt{FPR}},{\theta}_b^{\texttt{FPR}}, \beta)> 0$. Therefore, to attain  $h^\texttt{FPR}(\hat{\theta}_a^{\texttt{TPR}},\hat{\theta}_b^{\texttt{FPR}}, \beta)=0$, at least one of the thresholds should increase compared to the unbiased thresholds $\{{\theta}_a^{\texttt{TPR}},{\theta}_b^{\texttt{TPR}}\}$. In addition, as $\hat{f}^0_g(x)=f^0_g(x)$, and by the definition of \texttt{FPR}, the changes in the thresholds are aligned (i.e, either both decrease or both increase compared to the unbiased case). Therefore, $\hat{\theta}^{\texttt{FPR}}_g(\beta)\geq {\theta}^{\texttt{FPR}}_g$ for both groups.

    $\bullet$ For \texttt{TPR}: from Table~\ref{t:fair-thds}, The firm picks the thresholds such that $h^\texttt{TPR}(\hat{\theta}_a^{\texttt{TPR}},\hat{\theta}_b^{\texttt{TPR}}, \beta)=0$, where \[h^\texttt{TPR}({\theta}_a,{\theta}_b, \beta) : = \frac{n_a((1-\beta)\alpha_a+\beta)}{(1-\beta)\gamma_a(\theta_a)+\beta} + \frac{n_b\alpha_b}{\gamma_b(\theta_b)} - \tfrac{1}{\frac{u_-}{u_++u_-}}(n_a((1-\beta)\alpha_a+\beta)+n_b\alpha_b)~.\]
    
        Note that as $\gamma_g(x)$ are increasing functions under Assumption~\ref{ass:MLR}, $h^\texttt{TPR}$ is decreasing in both thresholds. To identify its trend in $\beta$, 
    the derivative of $h^\texttt{TPR}$ with respect to $\beta$ is given by
    \begin{align*}
        \frac{\partial h^\texttt{TPR}}{\partial \beta} = &  n_a\frac{(1-\alpha_a)((1-\beta)\gamma_a(\theta_a)+\beta)-(1-\gamma_a(\theta_a))((1-\beta)\alpha_a+\beta)}{((1-\beta)\gamma_a(\theta_a)+\beta)^2} - \frac{n_a(1-\alpha_a)}{\frac{u_-}{u_++u_-}}\\
        & = n_a(1-\alpha_a) \left(\frac{1}{(1-\beta)\gamma_a(\theta_a)+\beta}-\frac{1}{\frac{u_-}{u_++u_-}}\right) + \frac{n_a(1-\gamma_a(\theta_a))((1-\beta)\alpha_a+\beta)}{((1-\beta)\gamma_a(\theta_a)+\beta)^2}
    \end{align*}
    As $h^\texttt{TPR}(\hat{\theta}^\texttt{TPR}_a, \hat{\theta}^\texttt{TPR}_b, \beta) = 0$, and $\gamma_b(x)\leq \gamma_a(x)$, we can conclude that $\frac{1}{\gamma_b(\hat{\theta}^\texttt{FPR}_b)}\leq \frac{1}{\frac{u_-}{u_++u_-}} \leq \frac{1}{(1-\beta)\gamma_a(\hat{\theta}^\texttt{FPR}_a)+\beta}$. This means that $\frac{\partial h^\texttt{FPR}}{\partial \beta}$ is increasing at the optimal biased thresholds for each $\beta\in(0,1)$. Therefore, $h^\texttt{TPR}({\theta}_a^{\texttt{TPR}},{\theta}_b^{\texttt{TPR}}, \beta)\leq 0$ as $\beta$ decreases from 1. That means that in order to attain $h^\texttt{TPR}(\hat{\theta}_a^{\texttt{TPR}},\hat{\theta}_b^{\texttt{TPR}}, \beta)=0$, at least one of the thresholds should increase compared to the unbiased thresholds $\{{\theta}_a^{\texttt{TPR}},{\theta}_b^{\texttt{TPR}}\}$. 
    In addition, from \eqref{eq:a-aafo-b} in the proof of Lemma~\ref{lemma:fair-opt}, we know that given that $\beta=1$, if $\theta^{\texttt{TPR}}_a$ drops, so should $\theta^{\texttt{TPR}}_b$ for the \texttt{TPR} constraint to continue to hold. 
    So it must be that both thresholds increase. We conclude that both thresholds should increase, that is $\hat{\theta}^{\texttt{TPR}}_g(\beta)\geq {\theta}^{\texttt{TPR}}_g$ for both groups when $\beta$ drops from 1.

Lastly, for all three constraints, applying the same argument to levels of qualification assessment biases $\beta^1>\beta^2$, we conclude that  $\hat{\theta}^{\texttt{f}}_g(\beta^2)>\hat{\theta}^{\texttt{f}}_g(\beta^1)$. That is, $\hat{\theta}^{\texttt{f}}_g(\beta)$ is decreasing in $\beta$.

\paragraph{Part (ii):} The \texttt{DP}-constrained thresholds when data is biased are selected such that $\int_{\hat{\theta}^{\texttt{DP}}_a}^\infty {f}_a(x)\mathrm{d}x=\int_{\hat{\theta}^{\texttt{DP}}_b}^\infty \hat{f}_b(x)\mathrm{d}x$ (i.e., based on the biased training data); since $\hat{f}_a(x)=f_a(x)$, this constraint is also satisfied on the unbiased data. That is, \texttt{DP} continues to hold (on the unbiased training data, as intended) at  $\{\hat{\theta}^{\texttt{DP}}_a,\hat{\theta}^{\texttt{DP}}_b\}$. 

Next, as $\hat{f}^0_g(x)= f^0_g(x)$ and  $\hat{f}^1_a(x)=\frac{\alpha_a}{(1-\beta)\alpha_a+\beta} \frac{(1-\beta)\gamma_a(x)+\beta}{\gamma_a(x)} f^1_b(x)$, the thresholds satisfying \texttt{FPR} on the biased data also satisfy \texttt{FPR} on the unbiased data (since it uses $\hat{f}^0_g$ in its assessments), while the same is not true for \texttt{TPR} (since it uses $\hat{f}^1_g$ in its assessments).  
\end{proof}

\subsection{Proof of Proposition~\ref{prop:f-bias-feature}}
\begin{proof}
Define $l_g(x):=\frac{f^1_g(x)}{f^0_g(x)}$. First, we note that when $\hat{l}_b(x)=\beta(x)l_b(x)$, this bias does not affect the labels, and therefore $\hat{\alpha}_g=\alpha_g$. We now proceed with the proof of the proposition. 

\emph{Part (i):} This part follows directly from the definition of the \texttt{TPR}/\texttt{FPR} fairness constraints. In words, all pairs of decision thresholds that satisfy the desired fairness notion remain valid as the problem primitive used to identify these thresholds is assumed to have remained unaffected by the feature measurement error. 

\emph{Part (ii):} For \texttt{TPR}: The firm picks the thresholds such that $h^\texttt{TPR}(\hat{\theta}_a^{\texttt{TPR}},\hat{\theta}_b^{\texttt{TPR}}, \beta(\hat{\theta}_b^{\texttt{TPR}}))=0$, where
\begin{align*}
    &h^\texttt{TPR}({\theta}_a,{\theta}_b, \beta(\theta_b)) \\
    &= \frac{n_a(1-\alpha_a)}{{l}_a(\theta_a)} + \frac{n_b(1-\alpha_b)}{\beta(\theta_b){l}_b(\theta_b)}-\frac{u_+}{u_-}(n_a\alpha_a+n_b\alpha_b)~.
\end{align*}
First, we note that by Assumption~\ref{ass:MLR}, the functions $l_g(x)$ are increasing. In addition, $\beta(x)$ is non-decreasing in $x$. As a result, $h^\texttt{FPR}$ is decreasing in both thresholds. Further, as the value of $\beta(\theta)$ increases, $h^\texttt{FPR}$ decreases. 

In the unbiased case, $\beta(x)=1$ for all $x$. Therefore, by definition $h^\texttt{TPR}({\theta}_a^{\texttt{TPR}},{\theta}_b^{\texttt{TPR}}, 1) = 0$. When data is biased, $\beta(x)<1$, and $h^\texttt{TPR}({\theta}_a^{\texttt{TPR}},{\theta}_b^{\texttt{FPR}}, \beta') > 0$ for $\beta'<1$. As a result, in order to get $h^\texttt{TPR}(\hat{\theta}_a^{\texttt{TPR}},\hat{\theta}_b^{\texttt{TPR}}, \beta(\hat{\theta}_b^{\texttt{TPR}}))=0$, at least one of the thresholds should increase compared to the unbiased thresholds $\{{\theta}_a^{\texttt{TPR}},{\theta}_b^{\texttt{TPR}}\}$.  

In addition, we know that $\int_{{\theta}^{\texttt{TPR}}_a}^\infty {f}^1_a(x)\mathrm{d}x=\int_{{\theta}^{\texttt{TPR}}_b}^\infty {f}^1_b(x)\mathrm{d}x$ from the unbiased case, while $\int_{\hat{\theta}^{\texttt{TPR}}_a}^\infty {f}^1_a(x)\mathrm{d}x=\int_{\hat{\theta}^{\texttt{TPR}}_b}^\infty \hat{f}^1_b(x)\mathrm{d}x$ in the biased case. Now assume $\hat{\theta}^{\texttt{TPR}}_a>{\theta}^{\texttt{TPR}}_a$; then, $\int_{\hat{\theta}^{\texttt{TPR}}_a}^\infty {f}^1_a(x)\mathrm{d}x < \int_{{\theta}^{\texttt{TPR}}_a}^\infty {f}^1_a(x)\mathrm{d}x$, and as a result $\int_{\hat{\theta}^{\texttt{TPR}}_b}^\infty \hat{f}^1_b(x)\mathrm{d}x < \int_{{\theta}^{\texttt{TPR}}_b}^\infty {f}^1_b(x)\mathrm{d}x$, or equivalently, that $\hat{F}^1_b(\hat{\theta}^{\texttt{TPR}}_b) >  F_b^1({\theta}^{\texttt{TPR}}_b)$. If $\hat{F}^1_b({\theta}^{\texttt{TPR}}_b)\leq F_b^1({\theta}^{\texttt{TPR}}_b)$, then this can only hold if $\hat{\theta}^{\texttt{TPR}}_b>{\theta}^{\texttt{TPR}}_b$. Otherwise, This relation can hold for either increasing or decreasing $\hat{\theta}^{\texttt{TPR}}_b$ relative to ${\theta}^{\texttt{TPR}}_b$.

To conclude, it is sufficient to have  $\hat{F}^1_b(x)< F^1_b(x), \forall x$ for both \texttt{TPR}-constrained decisions thresholds to increase relative to the unbiased case once data is biased.  In addition, as \linebreak  ${F}^1_a(\hat{\theta}_a^\texttt{TPR})=\hat{F}^1_b(\hat{\theta}_b^\texttt{TPR})< F_b^1(\hat{\theta}_b^\texttt{TPR})$, the fairness constraint is violated. 

\emph{Part (iii):} For \texttt{FPR}: The firm picks the thresholds such that $h^\texttt{FPR}(\hat{\theta}_a^{\texttt{FPR}},\hat{\theta}_b^{\texttt{FPR}}, \beta(\hat{\theta}_b^{\texttt{FPR}}))=0$, where
\begin{align*}
    &h^\texttt{FPR}({\theta}_a,{\theta}_b, \beta(\theta_b)) \\
    &= n_a\alpha_a {l}_a(\theta_a) + n_b\alpha_b\beta(\theta_b){l}_b(\theta_b)-  \frac{u_-}{u_+}(n_a(1-\alpha_a)+n_b(1-\alpha_b))~.
\end{align*}
First, we note that by Assumption~\ref{ass:MLR}, the functions $l_g(x)$ are increasing. In addition, $\beta(x)$ is non-decreasing in $x$. As a result, $h^\texttt{FPR}$ is increasing in both thresholds. Further, as the value of $\beta(\theta)$ increases, so does $h^\texttt{FPR}$. 

In the unbiased case, $\beta(x)=1$ for all $x$. Therefore, by definition $h^\texttt{FPR}({\theta}_a^{\texttt{FPR}},{\theta}_b^{\texttt{FPR}}, 1) = 0$. When data is biased, $\beta(x)<1$, and $h^\texttt{FPR}({\theta}_a^{\texttt{FPR}},{\theta}_b^{\texttt{FPR}}, \beta') < 0$ for $\beta'<1$. As a result, in order to get $h^\texttt{FPR}(\hat{\theta}_a^{\texttt{FPR}},\hat{\theta}_b^{\texttt{FPR}}, \beta(\hat{\theta}_b^{\texttt{FPR}}))=0$, at least one of the thresholds should increase compared to the unbiased thresholds $\{{\theta}_a^{\texttt{FPR}},{\theta}_b^{\texttt{FPR}}\}$.  

In addition, we know that $\int_{{\theta}^{\texttt{FPR}}_a}^\infty {f}^0_a(x)\mathrm{d}x=\int_{{\theta}^{\texttt{FPR}}_b}^\infty {f}^0_b(x)\mathrm{d}x$ from the unbiased case, while $\int_{\hat{\theta}^{\texttt{FPR}}_a}^\infty {f}^0_a(x)\mathrm{d}x=\int_{\hat{\theta}^{\texttt{FPR}}_b}^\infty \hat{f}^0_b(x)\mathrm{d}x$ in the biased case. Now assume $\hat{\theta}^{\texttt{FPR}}_a>{\theta}^{\texttt{FPR}}_a$; then, $\int_{\hat{\theta}^{\texttt{FPR}}_a}^\infty {f}^0_a(x)\mathrm{d}x < \int_{{\theta}^{\texttt{FPR}}_a}^\infty {f}^0_a(x)\mathrm{d}x$, and as a result $\int_{\hat{\theta}^{\texttt{FPR}}_b}^\infty \hat{f}^0_b(x)\mathrm{d}x < \int_{{\theta}^{\texttt{FPR}}_b}^\infty {f}^0_b(x)\mathrm{d}x$, or equivalently, that $\hat{F}^0_b(\hat{\theta}^{\texttt{FPR}}_b)> F_b^0({\theta}^{\texttt{FPR}}_b)$. If $\hat{F}^0_b({\theta}^{\texttt{FPR}}_b)\leq F_b^0({\theta}^{\texttt{FPR}}_b)$, then this can only hold if $\hat{\theta}^{\texttt{FPR}}_b>{\theta}^{\texttt{FPR}}_b$. Otherwise, This relation can hold for either increasing or decreasing $\hat{\theta}^{\texttt{FPR}}_b$ relative to ${\theta}^{\texttt{FPR}}_b$.

To conclude, it is sufficient to have  $\hat{F}^0_b(x)\leq F_b^0(x), \forall x$ so that both \texttt{FPR}-constrained decisions thresholds increase relative to the unbiased case once data is biased.  In addition, as \linebreak  ${F}^0_a(\hat{\theta}_a^\texttt{FPR})=\hat{F}^0_b(\hat{\theta}_b^\texttt{FPR})\leq F_b^0(\hat{\theta}_b^\texttt{FPR})$, the fairness constraint is in general violated.

\emph{Part (iv):}  For \texttt{DP}: The firm picks the thresholds such that $h^\texttt{DP}(\hat{\theta}_a^{\texttt{DP}},\hat{\theta}_b^{\texttt{DP}}, \beta(\hat{\theta}_b^{\texttt{DP}}))=0$, where 
    \[h^\texttt{DP}({\theta}_a,{\theta}_b, \beta(\theta_b)) = \sum_g n_g \frac{\hat{l}_g(\theta_g) -  \frac{u_-}{u_+}\frac{1-\alpha_g}{\alpha_g}}{\hat{l}_g(\theta_g) +  \frac{1-\alpha_g}{\alpha_g}}~.\]
    The derivative with respect to $\theta_g$ is given by
    \[\frac{\partial h^\texttt{DP}({\theta}_a,{\theta}_b, \beta)}{\partial \theta_g} = n_g \frac{\hat{l}'_g(\theta_g)(1 +  \frac{u_-}{u_+})\frac{1-\alpha_g}{\alpha_g}}{(\hat{l}_g(\theta_g) +  \frac{1-\alpha_g}{\alpha_g})^2}~,\]
    where $\hat{l}'(\cdot)$ denotes the derivative of $\hat{l}$ with respect to $x$. Note that as $\hat{l}_g(x)$ are increasing functions under Assumption~\ref{ass:MLR}, this means that $h^\texttt{DP}$ is increasing in both thresholds. Similarly, it is easy to check that the above is increasing in $\beta(\theta_b)$. The remainder of the proof proceeds with arguments similar to those of parts (ii) and (iii). 
    
\emph{Part (v):} Our analytical arguments in Parts (ii)-(iv) identify conditions under which this decrease is plausible. We provide numerical examples for the cases of \texttt{DP} and \texttt{TPR} in Section~\ref{sec:experiments}. 

\end{proof}

\subsection{Sensitivity of \texttt{DP/TPR/FPR} thresholds to feature measurement errors}\label{app:prop-f-bias-feature-sensitivity}

\begin{proposition}[Sensitivity of \texttt{TPR/FPR} thresholds to feature measurement errors]\label{prop:f-bias-feature-sensitivity}
Let $l_g(x):=\frac{f_g^1(x)}{f_g^0(x)}$. Assume the features of group $b$ are incorrectly measured, so that $\hat{l}_b(x, \beta)$, where $\beta\in(0,1]$ is the error rate and $\hat{l}_b(x,1)=l_b(x), \forall x$. Let ${\theta}^{\texttt{f}}_g$ and $\hat{\theta}^{\texttt{f}}_g(\beta)$ denote the optimal decision thresholds satisfying fairness constraint $\texttt{f}\in\{\texttt{TPR}, \texttt{FPR}\}$, obtained from unbiased data and data with biases on group $b$ with error rate $\beta$, respectively. Then, the rate of change of group $b$'s thresholds at $\beta=1$ are given by
\begin{align*}
    \frac{\partial \hat{\theta}^{\texttt{TPR}}_b(\beta)}{\partial \beta}|{_{\beta=1}} &= -\frac{\frac{\partial \hat{f}^1_b({\theta}_b^{\texttt{TPR}})}{\partial \beta}|{\beta=1}\frac{l_a'({\theta}_a^{\texttt{TPR}})}{f^1_a({\theta}_a^{\texttt{FPR}})}(\frac{l_b(\theta^{\texttt{TPR}}_b)}{l_a(\theta^{\texttt{TPR}}_a)})^2 +\frac{\partial \hat{l}_b({\theta}_b^{\texttt{TPR}})}{\partial \beta}|{\beta=1}}{\frac{n_a}{n_b}\frac{1-\alpha_a}{1-\alpha_b}\frac{f^1_b({\theta}_b^{\texttt{TPR}})}{f^1_a({\theta}_a^{\texttt{TPR}})}(\frac{l_b(\theta^{\texttt{TPR}}_b)}{l_a(\theta^{\texttt{TPR}}_a)})^2{l'_a({\theta}_a^{\texttt{TPR}})} + l'_b({\theta}_b^{\texttt{TPR}})}\\
    \frac{\partial \hat{\theta}^{\texttt{FPR}}_b(\beta)}{\partial \beta}|{_{\beta=1}} &= -\frac{\frac{\partial \hat{f}^0_b({\theta}_b^{\texttt{FPR}})}{\partial \beta}|{\beta=1}\frac{l_a'({\theta}_a^{\texttt{FPR}})}{f^0_a({\theta}_a^{\texttt{FPR}})} + \frac{\partial \hat{l}_b({\theta}_b^{\texttt{FPR}})}{\partial \beta}|{\beta=1}}{\frac{n_a}{n_b}\frac{\alpha_a}{\alpha_b}\frac{f^0_b({\theta}_b^{\texttt{FPR}})}{f^0_a({\theta}_a^{\texttt{FPR}})}{l'_a({\theta}_a^{\texttt{FPR}})} + l'_b({\theta}_b^{\texttt{FPR}})}
\end{align*}
where $l'(\cdot)$ is its first derivative with respect to $x$. 
Further, the rate of change of group $a$'s thresholds at $\beta=1$ are given by
\begin{align*}
    \frac{\partial \hat{\theta}^{\texttt{TPR}}_a(\beta)}{\partial \beta}|{_{\beta=1}} &= \frac{1}{f^1_a({\theta}_a^{\texttt{TPR}})} \frac{\hat{f}^1_b({\theta}_b^{\texttt{TPR}})}{\partial \beta}|{_{\beta=1}}\frac{\partial \hat{\theta}^{\texttt{TPR}}_b(\beta)}{\partial \beta}|{_{\beta=1}}\\
    \frac{\partial \hat{\theta}^{\texttt{FPR}}_a(\beta)}{\partial \beta}|{_{\beta=1}} &= \frac{1}{f^0_a({\theta}_a^{\texttt{FPR}})} \frac{\hat{f}^0_b({\theta}_b^{\texttt{FPR}})}{\partial \beta}|{_{\beta=1}}\frac{\partial \hat{\theta}^{\texttt{FPR}}_b(\beta)}{\partial \beta}|{_{\beta=1}}
\end{align*}
\end{proposition}

\begin{proof}
The proof of this proposition invokes the implicit function theorem in the same way as that of Proposition~\ref{prop:f-bias-gamma-sensitivity}. However, the main distinction is in that when $\hat{\theta}_a$ is expressed as a function of $\hat{\theta}_b$, the relation is also affected by the dependence on $\beta$. For instance, in the case of \texttt{FPR}, the fairness constraint requires that $F^0_a(\hat{\theta}_a^\texttt{FPR})=\hat{F}_b(\hat{\theta}_b^\texttt{FPR})$, or equivalently, that $\hat{\theta}_a^\texttt{FPR}=(F^0_a)^{-1}(\hat{F}_b(\hat{\theta}_b^\texttt{FPR}))$. As a result, 
\[\frac{\partial \hat{\theta}^\texttt{FPR}_a}{\partial \hat{\theta}^\texttt{FPR}_b} = \frac{\hat{f}^0_b(\hat{\theta}^\texttt{FPR}_b)}{f^0_a(\hat{\theta}^\texttt{TPR}_a)}~, ~~ \frac{\partial \hat{\theta}^\texttt{FPR}_a}{\partial \beta} = \frac{\partial \hat{f}^0_b(\hat{\theta}^\texttt{FPR}_b)}{\partial \beta}\frac{1}{f^0_a(\hat{\theta}^\texttt{TPR}_a)}~.\]
The remainder of the proof proceeds similarly to Proposition~\ref{prop:f-bias-gamma-sensitivity}. 
\end{proof}

The following corollary of Proposition~\ref{prop:f-bias-feature-sensitivity} compares the sensitivity of \texttt{TPR} and \texttt{FPR} to the same feature measurement errors. 

\begin{corollary}[\texttt{FPR} can be more sensitive to measurement errors than \texttt{TPR}]\label{cor:TPR-vs-FPR} Consider $\frac{\partial \hat{\theta}^{\texttt{f}}_b(1)}{\partial \beta}$, the rate of change of group $b$'s thresholds at $\beta=1$. Assume $f_a^y=f_b^y, \forall y$, and so, $\theta_g^\texttt{TPR}=\theta_g^\texttt{FPR}=\theta, ~\forall g$. Then, there exists a $\bar{\theta}$, such that for all instances with $\theta\geq \bar{\theta}$, we have $|\tfrac{\partial \hat{\theta}^{\texttt{TPR}}_b(1)}{\partial \beta}|<|\tfrac{\partial \hat{\theta}^{\texttt{FPR}}_b(1)}{\partial \beta}|$; that is, \texttt{FPR} is more sensitive to feature measurement errors than \texttt{TPR}. 
\end{corollary}

\end{document}